\documentclass{article} % For LaTeX2e
\usepackage{nips13submit_e,times}
\usepackage{hyperref}
\usepackage{url}
\title{Moment-based Uniform Deviation Bounds for $k$-means and Friends}

\author{
Matus Telgarsky
\qquad
\qquad
Sanjoy Dasgupta
\\
Computer Science and Engineering,
UC San Diego\\
\texttt{\{mtelgars,dasgupta\}@cs.ucsd.edu}
}

\nipsfinalcopy % Uncomment for camera-ready version

\usepackage[numbers,comma]{natbib}  %nips wants numbers
\usepackage{amsmath}
\usepackage{amssymb,graphicx,dsfont}

\usepackage{amsthm}
\usepackage{thmtools}
\usepackage{subfig} %must be after thmtools..
\usepackage{hyperref}
\usepackage{cleveref}
\usepackage{mathrsfs}

\def\1{\mathds 1}
\def\R{\mathbb R}

\def\bbE{\mathbb E}

\def\cF{\mathcal F}

\def\cH{\mathcal H}

\def\cM{\mathcal M}
\def\cN{\mathcal N}
\def\cO{\mathcal O}

\def\cS{\mathcal S}

\def\cX{\mathcal X}

\def\sfB{\mathsf B}

\def\nf{\nabla f}

\newcommand{\ip}[2]{\left\langle #1, #2 \right \rangle}
\newcommand{\argmin}{\operatornamewithlimits{arg\,min}}

\numberwithin{equation}{section}
\declaretheorem[numberlike=equation]{theorem}
\declaretheorem[numberlike=theorem]{lemma}
\declaretheorem[numberlike=theorem]{proposition}
\declaretheorem[numberlike=theorem]{corollary}
\declaretheoremstyle[%
qed={\ensuremath\blacksquare} %added this back in because the example looked weird..
]{exstyle}
\declaretheorem[numberlike=theorem]{definition}

\declaretheorem[numberlike=theorem,style=exstyle]{example}

\def\Hkm{\cH}
\def\Smog{{\cS_{\textup{mog}}}}
\def\phig{\phi_{\textup{g}}}

\begin{document}

\maketitle

\begin{abstract}
    Suppose $k$ centers are fit to $m$ points by
    heuristically minimizing the $k$-means cost;
    what is the corresponding fit over the source distribution?
    This question is resolved here for distributions with $p\geq 4$ bounded moments;
    in particular, the difference between the sample cost and distribution cost
    decays with $m$ and $p$ as $m^{\min\{-1/4, -1/2+2/p\}}$.
    The essential technical contribution is a mechanism
    to uniformly control deviations in the face of unbounded parameter sets,
    cost functions, and source distributions.
    To further demonstrate this mechanism,
    a soft clustering variant of $k$-means cost is also considered,
    namely
    the
    log likelihood of a Gaussian mixture,
    subject to the constraint that all covariance matrices have bounded spectrum.
    Lastly, a rate with refined constants is provided for $k$-means
    instances possessing some cluster structure.
   %XXX \red{need to check that $k$-means rate; should be same as MoG}
%
%
%%  Suppose a finite sample is heuristically fit with a set of means according to
%%  the $k$-means cost, or a mixture of Gaussians according to log likelihood;
%%  how good is the fit over the underlying distribution?
%%  This question is resolved here in the case of distributions with four bounded
%%  moments, first for $k$-means,
%%  and second for Gaussian mixtures with some specified lower bound
%%  on the eigenvalues of the covariance matrices.
%%  \red{\textbf{prove this, for the love of all that is holy.}
%%      In either case, the presence of higher moment bounds pulls the rate from
%%  $1/m^{1/4}$ towards $1/\sqrt{m}$.}
%%  The essential technical contribution is a mechanism
%%  to uniformly control deviations in the face of unbounded parameter sets,
%%  cost functions, and source distributions.
%%  For $k$-means, a refinement is also given with some sensitivity to
%%  the clusterability of the source distribution.
\end{abstract}

\section{Introduction}

Suppose a set of $k$ centers $\{p_i\}_{i=1}^k$
is selected by approximate minimization of $k$-means
cost; how does the fit
over the sample compare with the fit over the distribution?
Concretely: given $m$ points sampled from a source distribution $\rho$,
what can be said about the quantities
\begin{align}
    \Bigg| %lel can't put \left due to the ampersand inside
        \frac 1 m \sum_{j=1}^m \min_{i} \|x_j-p_i\|_2^2
        &-
        \int \min_{i} \|x-p_i\|_2^2 d\rho(x)
        \Bigg|
        &&\textup{($k$-means),}
        \label{eq:intro:km}
        \\
        \Bigg|
        \frac 1 m \sum_{j=1}^m \ln\left(\sum_{i=1}^k \alpha_i p_{\theta_i}(x_j)\right)
        &-
        \int
        \ln\left(\sum_{i=1}^k \alpha_i p_{\theta_i}(x)\right)
        d\rho(x)
        \Bigg|
        %&&\textup{(Gaussian mixture),}
        &&\textup{(soft $k$-means),}
        \label{eq:intro:mog}
\end{align}
%where $\rho$ denotes the source probability measure,
where %in the second problem, a \emph{soft clustering} variant of $k$-means,
each $p_{\theta_i}$ denotes the density of a Gaussian with a covariance
matrix whose eigenvalues lie in some closed positive interval.

The
literature offers a wealth of information related to this question.  For $k$-means, there
is firstly a consistency result:
under some identifiability conditions, the global minimizer over the sample will
converge to the global minimizer over the distribution as the sample size $m$ increases
\citep{pollard_km_cons}.  Furthermore, if the distribution is bounded, standard
tools can provide deviation inequalities \citep{Lugosi94ratesof,
Ben-david04aframework,
DBLP:conf/nips/RakhlinC06}.
For the second problem, which is maximum likelihood of a Gaussian mixture
(thus amenable to EM \citep{duda_hart_stork}),
classical results
regarding the consistency of maximum likelihood again provide that, under some
identifiability conditions, the optimal solutions over the sample converge to
the optimum over the distribution \citep{ferguson_large_sample_theory}.

The task here is thus: to provide finite sample guarantees for these problems,
but eschewing boundedness, subgaussianity, and similar assumptions in favor
of moment assumptions.

\subsection{Contribution}

%%  \red{it is not clear that MoGs are an auxiliary result.  If some comes in guns blazing
%%  wanting that to be the point and raises hell over the assumptions, I am screwed.}

%%  \blue{also stress that showing off the power of moment bounds here.}

The results here are of the following form:
given $m$ examples from a distribution with a few bounded moments,
and any set of parameters beating some fixed cost $c$,
the corresponding deviations in cost (as in \cref{eq:intro:km} and \cref{eq:intro:mog})
approach $\cO(m^{-1/2})$ with the availability of higher moments.
\begin{itemize}
    \item
        In the case of $k$-means (cf. \Cref{fact:km:basic:kmeans_cost}),
        $p\geq 4$ moments suffice, and the rate is $\cO(m^{\min\{-1/4, -1/2 + 2/p\}})$.
        For Gaussian mixtures (cf. \Cref{fact:mog:basic}),
        $p\geq 8$ moments suffice, and the rate is $\cO(m^{-1/2 + 3/p})$.
    \item
        The parameter $c$ allows these guarantees to hold for heuristics.  For instance,
        suppose $k$ centers are output by Lloyd's method.  While Lloyd's method
        carries no optimality guarantees, the results here hold for the output of Lloyd's
        method simply by setting $c$ to be the variance of the data, equivalently
        the $k$-means cost with a single center placed at the mean.
    \item
        The $k$-means and Gaussian mixture costs are only well-defined when the
        source distribution has $p\geq 2$ moments.  The condition of $p\geq 4$
        moments, meaning the variance has a variance, allows consideration
        of many heavy-tailed
        distributions, which are ruled out by boundedness and subgaussianity assumptions.
\end{itemize}

The main technical byproduct of the proof
is a mechanism to deal with the unboundedness of the
cost function; this technique will be detailed in \Cref{sec:km}, but the difficulty
and its resolution can be easily sketched here.

For a single set of centers $P$, the deviations in \cref{eq:intro:km} may be controlled
with an application of Chebyshev's inequality.  But this does not immediately grant
deviation bounds on another set of centers $P'$, even if $P$ and $P'$ are very close:
for instance, the difference between the two costs will grow as successively farther
and farther away points are considered.

The resolution is to simply note that there is so little probability mass in those far
reaches that the cost there is irrelevant.
Consider a single center $p$ (and assume $x\mapsto \|x-p\|_2^2$ is integrable);
the dominated convergence theorem grants
\[
    \int_{B_i} \|x-p\|_2^2 d\rho(x)
    \quad \to \quad
    \int \|x-p\|_2^2 d\rho(x),
    \qquad\qquad
    \textup{where $B_i := \{x \in \R^d : \|x-p\|_2 \leq i\}$.}
\]
In other words, a ball $B_i$ may be chosen so that
$\int_{B_i^c}\|x-p\|_2^2d\rho(x)\leq 1/1024$.
Now consider some $p'$ with $\|p-p'\|_2\leq i$.
Then
\[
    \int_{B_i^c} \|x-p'\|_2^2d\rho(x)
    \leq \int_{B_i^c} (\|x-p\|_2 + \|p-p'\|_2)^2d\rho(x)
    \leq 4\int_{B_i^c} \|x-p\|_2^2d\rho(x)
    \leq \frac 1 {256}.
\]
In this way, a single center may control the outer deviations
of whole swaths of other centers.
Indeed, those choices outperforming the reference score $c$ will provide
a suitable swath.  Of course, it would be nice to get a sense of the size
of $B_i$; this however is provided by the moment assumptions.

The general strategy is thus to split consideration
into
outer deviations, and local deviations.  The local deviations may be controlled by
standard techniques.  To control outer deviations, a single pair of dominating costs
--- a lower bound and an upper bound --- is controlled.

%   But even for a set of centers $P'$ which
%   is extremely close to $P$, the b
%   As the primary contribution here is perhaps the technique for doing so, it is useful
%   to describe the core difficulty, and sketch how to overcome it.  The two
%   main difficulties are: the data is unbounded, and the parameter space
%   is unbounded.  The data being unbounded can be handled for instance by using Chebyshev's
%   inequality rather than Hoeffding's inequality (and indeed this proves effective here).

%   Regarding the unbounded parameter space, suppose probability measure $\rho$ has
%   mean zero and unit variance; then $\int \|x\|_2^2 d\rho(x) = 1$.  From the perspective
%   of deviation theory, this expression is already interesting, as the integrand is, for
%   instance, neither bounded nor Lipschitz, and $\rho$ is similarly neither bounded
%   nor Gaussian
%   \red{cite  \citet[Theorem 5.6]{blm_conc} for Lipschitz-Gaussian concentration?}.
%   Digging deeper, the dominated convergence theorem provides
%   \[
%       \int_{B_i} \|x\|_2^2 d\rho(x)
%       \quad \to \quad
%       \int \|x\|_2^2 d\rho(x),
%       \qquad\qquad
%       \textup{where $B_i := \{x \in \R^d : \|x\|_2 \leq i\}$.}
%   \]
%   In other words, there exists a ball $B$ so that $\int_{B^c} \|x\|_2^2d\rho(x)$ is
%   arbitrarily close to zero.  More brazenly: \emph{effectively nothing} occurs in $B^c$,
%   even though $\|x\|_2^2$ is growing rapidly!

This technique can be found in the proof of the consistency of $k$-means due
to \citet{pollard_km_cons}.  The present work shows it can
also provide
finite sample guarantees, and moreover be applied outside hard clustering.
%   is not hard clustering (concretely, Gaussian mixture modeling with bounded covariance
%   spectra).

%%  \blue{uh well in the following need to mention the whole controlling-many-with-one
%%  thing.  yeah. can maybe even show the mechanism (squared triangle inequality thing..)}
%%  This means the data may be truncated
%%  down to a bounded set, and the centers may also be truncated (since only those closer
%%  to the action matter).  This idea was used in \citeauthor{pollard_km_cons}'s beautiful
%%  proof of the consistency of $k$-means clustering \citep{pollard_km_cons}: throughout
%%  the proof, one can constantly find points being compared to a single center at the origin.
%%  This in turn relates to the first part of the guarantee sketched above: given some
%%  reference score $c$, it can be shown that \emph{some} center must lie in a reasonable
%%  bounded set.  This technique is used to control both the $k$-means and Gaussian mixture
%%  costs.

%%  \red{say something about how 4 bounded moments not bad?  and how you need two bounded
%%      moments for the above things to make any sense?  thus just doing chebyshev on quantity
%%  care about?}

%%  \blue{somewhere stress that only need to control for ONE bracketing element.
%%  empiriciall process guys have the whole bracketing entropy crap.}

The content here is organized as follows.  The remainder of the
introduction surveys related
work, and subsequently \Cref{sec:setup} establishes some basic notation.  The core
deviation technique, termed \emph{outer bracketing} (to connect it to the bracketing
technique from empirical process theory), is presented along with the deviations
of $k$-means in \Cref{sec:km}.
The technique is then applied in \Cref{sec:mog} to a soft clustering variant, namely
log likelihood of Gaussian mixtures having bounded spectra.
%%      as will be discussed later,
%%      the assumption of bounded spectra can be violated in practice, and consequently
%%      this section mainly exists to demonstrate the application of outer bracketing in a
%%      distinct setting.
%%      \red{should I just postpone this discussion?}
%   while the assumption of bounded spectra is not guaranteed in practice in general,
%   The section is applied to mixtures of Gaussians in \Cref{sec:mog}, where,
%   as stated previously, the covariance matrices are assumed to have eigenvalues falling
%   in some bounded interval excluding zero.
%   %   to mixtures of Gaussians in \Cref{sec:mog}.
%   %   The guarantees for Gaussian mixtures take in a parameter $\sigma$, and the Gaussians
%   %   considered must have covariance $\varSigma$ satisfying $\Sigma\succeq \sigma I$;
%   %   this condition is in turn maintained by many practical implementations, which for instance
%   %   place a prior on the covariance, or simply prevent updates resulting a eigenvalues
%   %   below some threshold.
As a reprieve between these two heavier bracketing sections,
\Cref{sec:km:clamp} provides a simple refinement for $k$-means
which can adapt to cluster structure.

All proofs are deferred to the appendices, however the construction and application
of outer brackets is sketched in the text.

\subsection{Related Work}

As referenced earlier, \citeauthor{pollard_km_cons}'s work deserves special mention,
both since it can be seen as the origin of the outer bracketing technique, and since it
handled $k$-means under similarly slight assumptions (just two moments, rather than the
four here) \citep{pollard_km_cons,pollard_km_clt}.  The present work hopes
to be a
spiritual successor, providing finite sample guarantees, and adapting 
technique to a soft clustering problem.

In the machine learning community, statistical guarantees for clustering have been
extensively studied under the topic of \emph{clustering stability}
\citep{ DBLP:conf/nips/RakhlinC06, Ben-david06asober, Shamir82clusterstability,
Shamir_modelselection}.  One formulation of stability is: if parameters are learned
over two samples, how close are they?  The technical component of these works frequently
involves finite sample guarantees, which in the works listed here make a boundedness
assumption, or something similar (for instance, the work of
\citet{Shamir82clusterstability} requires the cost function to satisfy
a bounded differences condition).  Amongst these finite sample guarantees,
the finite sample guarantees due to \citet{DBLP:conf/nips/RakhlinC06} are similar
to the development here \emph{after} the invocation of the outer bracket: namely,
a covering argument controls deviations over a bounded set.
The results of \citet{Shamir_modelselection} do not make a boundedness assumption,
but the main results are not finite sample guarantees; in particular, they rely on
asymptotic results due to~\citet{pollard_km_clt}.

There are many standard tools which may be applied to the problems here,
particularly if a boundedness assumption is made \citep{bbl_esaim,blm_conc};
for instance, \citet{Lugosi94ratesof} use tools from VC theory to handle
$k$-means in the bounded case.
Another interesting work, by \citet{Ben-david04aframework}, develops specialized
tools to measure the complexity of certain clustering problems; when applied
to the problems of the type considered here, a boundedness assumption is made.
%%  Therein it is remarked that unboundedness can lead to arbitrarily bad deviations
%%  \citet[Page 5 above Definition 2]{Ben-david04aframework}, thus the conditions here
%%  of four bounded mounts and some upper bound on cost of considered solutions
%%  can be seen as sufficient to circumvent those difficulties.
%As discussed previously, standard arguments (covering) are employed here
%to control the distribution
%after applying outer bracketing.

A few of the above works provide some negative results and related commentary on
the topic of uniform deviations for distributions with unbounded support
\citep[Theorem 3 and subsequent discussion]{Shamir_modelselection}
\citep[Page 5 above Definition 2]{Ben-david04aframework}.
The primary ``loophole'' here is to constrain consideration to those solutions beating
some reference score $c$.
It is reasonable to guess that such a condition entails that
a few centers must lie near the bulk of the distribution's mass;
making this guess rigorous is the first step here
both for $k$-means and for Gaussian mixtures, and moreover the same consequence was
used by \citeauthor{pollard_km_cons} for the consistency of $k$-means
\citep{pollard_km_cons}.  In \citeauthor{pollard_km_cons}'s work, only optimal choices
were considered, but the same argument relaxes to arbitrary $c$, which can thus encapsulate
heuristic schemes, and not just nearly optimal ones.  (The secondary loophole is
to make moment assumptions; these sufficiently constrain the structure of the distribution
to provide rates.)

In recent years, the empirical process theory community has produced a large
body of work on the topic of maximum likelihood
(see for instance the excellent overviews and recent work of
\citet{wellner_uw_overview,vdw_wellner,wellner_kmonotone_mle}).
As stated previously, the choice of the term ``bracket'' is to connect to empirical
process theory.  Loosely stated, a bracket is simply a pair of functions which sandwich
some set of functions; the \emph{bracketing entropy} is then (the logarithm of)
the number of brackets
needed to control a particular set of functions.
In the present work, brackets are paired with sets
which identify the far away regions they are meant to control; furthermore, while
there is potential for the use of many outer brackets, the approach here is able to
make use of just a single outer bracket.
The name bracket is suitable, as opposed to cover, since the bracketing elements need
not be members of the function class being dominated.  (By contrast,
\citeauthor{pollard_km_cons}'s use in the
proof of the consistency of $k$-means was more akin to covering,
in that remote fluctuations were compared to that of a
a single center placed at the origin
\citep{pollard_km_cons}.)

\section{Notation}
\label{sec:setup}
The ambient space will always be the Euclidean space $\R^d$, though a few
results will be stated for a general domain $\cX$.  The source probability
measure will be $\rho$, and when a finite sample of size $m$ is available,
$\hat\rho$ is the corresponding empirical measure.  Occasionally, the variable
$\nu$ will refer to an arbitrary probability measure (where $\rho$ and $\hat\rho$ will serve
as relevant instantiations).  Both integral and expectation notation will be used;
for example, $\bbE(f(X)) = \bbE_\rho(f(X) = \int f(x)d\rho(x)$;
for integrals, $\int_B f(x)d\rho(x) = \int f(x) \1[x\in B] d\rho(x)$, where $\1$ is
the indicator function.
The moments of $\rho$ are defined as follows.

\begin{definition}
    Probability measure $\rho$ has
    \emph{order-$p$ moment bound $M$ with respect to norm $\|\cdot\|$}
    when $\bbE_\rho\|X-\bbE_\rho(X)\|^l\leq M$ for $1\leq l\leq p$.
\end{definition}

For example, the typical setting of $k$-means uses norm $\|\cdot\|_2$,
and at least two moments are needed for the cost over $\rho$ to be finite; the condition
here of needing 4 moments can be seen as naturally arising via Chebyshev's inequality.
Of course, the availability of higher moments is beneficial,
dropping the rates here from $m^{-1/4}$ down to $m^{-1/2}$.
Note that the basic controls derived from moments, which are primarily elaborations
of Chebyshev's inequality, can be found in \Cref{sec:moments}.

The $k$-means analysis will generalize slightly beyond the single-center
cost $x\mapsto\|x-p\|_2^2$ via \emph{Bregman divergences}
\citep{censor_zenios,bregman_clustering}.
\begin{definition}
    Given a convex differentiable function $f :\cX\to \R$,
    the corresponding \emph{Bregman divergence}
    is $\sfB_f(x,y) := f(x) - f(y) - \ip{\nf(y)}{x-y}$.
\end{definition}
Not all Bregman divergences are handled; rather, the following regularity conditions
will be placed on the convex function.
\begin{definition}
    A convex differentiable function $f$ is \emph{strongly convex} with modulus $r_1$
    and has \emph{Lipschitz gradients} with constant $r_2$, both respect to some
    norm $\|\cdot\|$, when $f$ (respectively) satisfies
    \begin{align*}
        f(\alpha x + (1-\alpha) y)
        & \leq
        \alpha f(x)
        + (1-\alpha) f(y)
        - \frac{r_1\alpha(1-\alpha)}{2} \|x-y\|^2,
        \\
        \|\nf(x) - \nf(y)\|_* &\leq r_2\|x-y\|,
    \end{align*}
    where $x,y \in \cX$, $\alpha\in[0,1]$, and $\|\cdot\|_*$ is the dual of
    $\|\cdot\|$. (The Lipschitz gradient condition is sometimes
    called \emph{strong smoothness}.)
  %%For simplicity, $r_1$ and $r_2$ will be referred to
  %%as the \emph{convexity constants of $f$}.
\end{definition}

These conditions are a fancy way of saying the corresponding Bregman divergence is
sandwiched between two quadratics (cf. \Cref{fact:bregman:easy_norms}).

\begin{definition}
    Given a convex differentiable function $f:\R^d\to R$ which is strongly convex and
    has Lipschitz gradients with respective constants $r_1,r_2$ with respect
    to norm $\|\cdot\|$,
    the \emph{hard $k$-means cost} of a single point $x$
    according to a set of centers $P$ is
    \[
        \phi_f(x;P) := \min_{p\in P} \sfB_f(x, p).
    \]
    The corresponding $k$-means cost of a set of points (or distribution)
    is thus computed as $\bbE_\nu(\phi_f(X;P))$,
    and let $\Hkm_f(\nu;c,k)$ denote all sets of at most $k$ centers beating cost $c$,
    meaning
    \[
        \Hkm_f(\nu;c,k) := \{ P :
            |P| \leq k,
            \bbE_\nu(\phi_f(X;P)) \leq c
        \}.
    \]
 %% $k$-means cost o
 %% the corresponding class of \emph{hard $k$-means costs} $\Hkm(f,k)$
 %% is all maps of the form
 %% \[
 %%     \phi(x;P) := \min_{p\in P} \sfB_f(x,P),
 %% \]
 %% where $P$ is a nonempty set of at most $k$ centers.
 %% Additionally, when measures $\rho$ and $\hat\rho$ are
 %% available, let $\Hkm(f,k;c)$ be those maps with cost at most $c$ according
 %% to $\rho$ or $\hat\rho$, meaning
 %% \[
 %%     \Hkm(f,k;c) := \{ \phi \in \Hkm(f,k) :
 %%         \bbE_\rho(\phi(X)) \leq c \lor \bbE_{\hat\rho}(\phi(X)) \leq c
 %%     \}.
 %% \]
\end{definition}

For example, choosing norm $\|\cdot\|_2$ and convex function $f(x)=\|x\|_2^2$ (which
has $r_1=r_2=2$), the corresponding Bregman divergence
is $\sfB_f(x,y) = \|x-y\|_2^2$,
and $\bbE_{\hat\rho}(\phi_f(X;P))$
denotes the vanilla $k$-means cost of some finite point set encoded in the
empirical measure $\hat\rho$.
%   the class $\Hkm(f,k)$ is vanilla $k$-means;
%   in particular,
%   the $k$-means cost of centers $P$ over a sample
%   is $\bbE_{\hat\rho}(\phi(X;P))$, and the cost over the source distribution
%   is $\bbE_{\rho}(\phi(X;P))$.

The hard clustering guarantees will work with $\Hkm_f(\nu;c,k)$,
where $\nu$ can be either the source distribution $\rho$, or its empirical
counterpart $\hat\rho$.  As discussed previously, it is reasonable to set $c$ to
simply the sample variance of the data, or a related estimate of the
true variance (cf. \Cref{sec:moments}).

%%  Crucially, the class $\Hkm(f,k;c)$ contains choices which are good both over a sample,
%%  and over a distribution; consequently, $c$ can be chosen as something easy to compute
%%  and easy to beat, for instance
%%  $c:=\bbE_{\hat\rho}(\phi(X;\{0\}))$,
%%  the cost of a single center placed at the origin.
%%  %%  \red{something with the randomness
%%  %%  seems wrong here.  like i'm sampling twice.  but this is the statement i want.}

Lastly, the class of Gaussian mixture penalties is as follows.
%%  As stated previously, a lower bound on the smallest eigenvalue of the covariance matrix
%%  is assumed; this is consistent with practical implementations, which for instance
%%  place a prior on the covariance, or simply preventing updated which make it too small.

%%  \begin{definition}
%%      The class of \emph{Gaussian mixture penalties} $\Smog(\sigma, k)$ consists of
%%      all functions of the form
%%      \[
%%          \phi(x; (\alpha, \Theta))
%%          := \phi(x; \{(\alpha_i,\theta_i) = (\alpha_i, \mu_i, \varSigma_i)\}_{i=1}^k)
%%          := \ln\left(
%%              \sum_{i=1}^k \alpha_i p_{\theta_i}(x)
%%          \right),
%%      \]
%%      where $p_{\theta_i}$ denotes a Gaussian density
%%      \[
%%          p_{\theta_i}(x) =
%%          \frac 1 {\sqrt{(2\pi)^d |\varSigma|}}
%%              \exp\left(-\frac 1 2 (x-\mu_i)^T \varSigma_i^{-1} (x-\mu_i)\right),
%%      \]
%%      and $\alpha\geq 0$, $\sum_i \alpha_i = 1$, and $\varSigma_i \succeq \sigma I$.
%%      Additionally, when measures $\rho$ and $\hat\rho$ are
%%      available, let $\Smog(\sigma,k;c)$ be those maps with cost at least $c$ according
%%      to $\rho$ or $\hat\rho$, meaning
%%      \[
%%          \Smog(\sigma,k;c) := \{ \phi \in \Smog(\sigma,k) :
%%              \bbE_\rho(\phi(X)) \leq c \lor \bbE_{\hat\rho}(\phi(X)) \leq c
%%          \}.
%%      \]
%%  \end{definition}
\begin{definition}
    Given Gaussian parameters $\theta := (\mu, \varSigma)$, % \in \R^d \times \R^{d^2}$,
    let $p_\theta$ denote Gaussian density
    \[
        p_{\theta}(x) =
        \frac 1 {\sqrt{(2\pi)^d |\varSigma_i|}}
            \exp\left(-\frac 1 2 (x-\mu_i)^T \varSigma_i^{-1} (x-\mu_i)\right).
    \]
    Given Gaussian mixture parameters $(\alpha, \Theta) = (\{\alpha_i\}_{i=1}^k,
    \{\theta_i\}_{i=1}^k)$ with $\alpha \geq 0$ and $\sum_i \alpha_i = 1$
    (written $\alpha \in \Delta$),
    the Gaussian mixture cost at a point $x$ is
    \[
        \phi_{\textup{g}}(x; (\alpha, \Theta))
        := \phi_{\textup{g}}(x; \{(\alpha_i,\theta_i) = (\alpha_i, \mu_i, \varSigma_i)\}_{i=1}^k)
        := \ln\left(
            \sum_{i=1}^k \alpha_i p_{\theta_i}(x)
        \right),
    \]
    Lastly, given a measure $\nu$, bound $k$ on the number of mixture parameters,
    and spectrum bounds $0 < \sigma_1\leq\sigma_2$,
    let $\Smog(\nu; c,k, \sigma_1, \sigma_2)$ denote those mixture parameters beating
    cost $c$, meaning
    \begin{align*}
        \Smog(\nu; c,k, \sigma_1, \sigma_2)
        &:=
        \left\{
            (\alpha, \Theta)
           %= (x; \{(\alpha_i,\theta_i)
           %= (\alpha_i, \mu_i, \varSigma_i)\}_{i=1}^{k'})
           %\\
           %&\qquad
            :
            %k' \leq k,
            \sigma_1 I \preceq \varSigma_i \preceq \sigma_2 I,
           %\alpha \geq 0,
           %\sum_i \alpha_i = 1,
            |\alpha| \leq k,
            \alpha\in\Delta,
            \bbE_\nu\left(\phi_{\textup{g}}(X; (\alpha, \Theta))\right) \leq c
        \right\}.
    \end{align*}
 %  The class of \emph{Gaussian mixture penalties} $\Smog(\sigma_1, \sigma_2, k)$
 %  consists of all functions of the form
 %  \[
 %      \phi(x; (\alpha, \Theta))
 %      := \phi(x; \{(\alpha_i,\theta_i) = (\alpha_i, \mu_i, \varSigma_i)\}_{i=1}^k)
 %      := \ln\left(
 %          \sum_{i=1}^k \alpha_i p_{\theta_i}(x)
 %      \right),
 %  \]
 %  where $p_{\theta_i}$ denotes a Gaussian density
 %  \[
 %      p_{\theta_i}(x) =
 %      \frac 1 {\sqrt{(2\pi)^d |\varSigma_i|}}
 %          \exp\left(-\frac 1 2 (x-\mu_i)^T \varSigma_i^{-1} (x-\mu_i)\right),
 %  \]
 %  and $\alpha\geq 0$, $\sum_i \alpha_i = 1$, and $\sigma_1 I \preceq \varSigma \preceq
 %  \sigma_2 I$, where $0 < \sigma_1 \leq \sigma_2$.
 %  Additionally, when measures $\rho$ and $\hat\rho$ are
 %  available, let $\Smog(\sigma_1,\sigma_2,k;c)$ be those maps with cost at least $c$ according
 %  to $\rho$ or $\hat\rho$, meaning
 %  \[
 %      \Smog(\sigma_1,\sigma_2,k;c) := \{ \phi \in \Smog(\sigma_1,\sigma_2,k) :
 %          \bbE_\rho(\phi(X)) \leq c \lor \bbE_{\hat\rho}(\phi(X)) \leq c
 %      \}.
 %  \]
\end{definition}

While a condition of the form $\varSigma \succeq \sigma_1 I$ is typically enforced in
practice (say, with a Bayesian prior, or by ignoring updates which shrink the covariance
beyond this point),
the condition $\varSigma \preceq \sigma_2 I$ is potentially violated.
These conditions will be discussed further in \Cref{sec:mog}.

\section{Controlling $k$-means with an Outer Bracket}
\label{sec:km}

First consider the special case of $k$-means cost.

\begin{corollary}
    \label{fact:km:basic:kmeans_cost}
    Set $f(x) := \|x\|_2^2$,
    %whereby $\Hkm(f,k)$ is the $k$-means cost.
    %whereby $\phi_f(\cdot;\cdot)$ is the $k$-means cost.
    whereby $\phi_f$ is the $k$-means cost.
    %Let reals $c$, $\epsilon > 0$,
    Let real $c\geq 0$
    and probability measure $\rho$ be given with order-$p$ moment bound
    $M$ with respect to $\|\cdot\|_2$,
    where $p\geq 4$ is a positive multiple of 4.
    Define the quantities
   %\red{was multiple of 4 necessary?  also, debate
   %completely removing $\phi$, $\Hkm$, etc.}
   %\blue{and uh is the exponent correct? mog doing better sounds wrong.}
    \begin{align*}
        c_1 := (2M)^{1/p} + \sqrt{2c},
        \quad
        M_1 := M^{1/(p-2)} + M^{2/p},
        \quad
        N_1 := 2 + 576d(c_1 + c_1^2 + M_1 + M_1^2).
    \end{align*}
    Then with probability at least $1-3\delta$ over the draw of a sample
    of size $m \geq \max\{(p / (2^{p/4+2}e))^2,9\ln(1/\delta)\}$, every set of centers
    %with $\phi(\cdot;P)\in \Hkm(f,k;c)$ satisfies
    $P \in \Hkm_f(\hat\rho;c,k)\cup \Hkm_f(\rho;c,k)$ satisfies
    \begin{align*}
        &\left|
        \int \phi_f(x;P)d\rho(x)
        - \int \phi_f(x;P)d\hat\rho(x)
        \right|
        \\
        %&\qquad\leq \frac {1}{m^{1/2 - 2(p-2)/p^2}}
        &\qquad\leq m^{-1/2 + \min\{1/4,2/p\}}
        \left(
            4 +
            (72c_1^2 + 32M_1^2)
            \sqrt{
                \frac 1 2 
            \ln
            \left(
                \frac{ (mN_1)^{dk} }{\delta}
            \right)
            }
            + \sqrt{\frac{2^{p/4}ep}{8m^{1/2}}} \left(\frac 2 \delta\right)^{4/p}
        \right).
    \end{align*}
\end{corollary}

%   As discussed previously, the two main characteristics the use of a reference score $c$,
%   and a rate which improves from $m^{-1/4}$ to $m^{-1/2}$ with the availability of higher
%   moments.

One artifact of the moment approach (cf. \Cref{sec:moments}),
heretofore ignored,
is the term $(2/\delta)^{4/p}$.
While this may seem inferior to $\ln(2/\delta)$, note that the choice
$p = 4\ln(2/\delta)/ \ln(\ln(2/\delta))$ suffices to make the two equal.

Next consider a general bound for Bregman divergences.
%%  ,
%%  meaning the set of centers $\Hkm_f(\nu;c,k)$.
This bound has a few
more parameters than \Cref{fact:km:basic:kmeans_cost}.  In particular,
the term $\epsilon$, which is instantiated to $m^{-1/2+1/p}$ in the proof of
\Cref{fact:km:basic:kmeans_cost}, catches the mass of points discarded due to the outer
bracket, as well as the resolution of the (inner) cover.
The parameter $p'$, which controls the tradeoff between $m$ and $1/\delta$, is set
to $p/4$ in the proof of \Cref{fact:km:basic:kmeans_cost}.

\begin{theorem}
    \label{fact:km:basic}
    Fix a reference norm $\|\cdot\|$ throughout the following.
    Let probability measure $\rho$ be given with order-$p$ moment bound
    $M$ where $p\geq 4$,
    a convex function $f$ with corresponding constants $r_1$ and $r_2$,
    reals $c$ and $\epsilon>0$,
    and integer $1\leq p' \leq p/2-1$ be given.
    Define the quantities
    \begin{align*}
      %%p'
      %%&:=
      %%p/2-1,
      %%\\
  %%%   M'
  %%%   &:= 2^{p'}\epsilon,
  %%%   \\
        R_B &:=
        \max\left\{
            (2M)^{1/p}
            + \sqrt{4c/r_1}
            ,
            \max_{i\in [p']}
            (M/\epsilon)^{1/(p-2i)}
        \right\},
        \\
        R_C &:=
        \sqrt{r_2/r_1}\left(
            (2M)^{1/p} + \sqrt{4c/r_1}
            +
        R_B\right)
        + R_B,
        \\
        B&:= \left\{ x \in \R^d : \|x-\bbE(X)\| \leq R_B\right\},
        \\
        C&:= \left\{ x \in \R^d : \|x-\bbE(X)\| \leq R_C\right\},
        \\
        \tau
        &:= \min\left\{
            \sqrt{\frac {\epsilon}{2r_2}}
            ,
            \frac {\epsilon}{2(R_B+R_C)r_2}
        \right\},
    \end{align*}
    and let $\cN$ be a cover of $C$ by $\|\cdot\|$-balls
    with radius $\tau$; in the case that $\|\cdot\|$ is an $l_p$ norm,
    the size of this cover has bound
    \[
        |\cN|
        \leq
        \left(1 + \frac {2R_Cd}{\tau}\right)^d.
    \]
    Then
    with probability at least $1-3\delta$ over the draw of a sample of size
    $m \geq \max\{p'/(e2^{p'}\epsilon), 9\ln(1/\delta)\}$,
    every set of centers
    %$P$ with $\phi(\cdot;P) \in \Hkm(f,k;c)$,
    $P \in \Hkm_f(\rho;c,k) \cup \Hkm_f(\hat\rho;c,k)$
    satisfies
    \[
        \left|
        \int \phi_f(x;P)d\rho(x)
        - \int \phi_f(x;P)d\hat\rho(x)
        \right|
        \leq 4\epsilon
        +
        4r_2R_C^2
        \sqrt{
            \frac {1}{2m}
            \ln
            \left(
                \frac {2|\cN|^k}{\delta}
            \right)
        }
        + \sqrt{\frac{e2^{p'}\epsilon p'}{2m}} \left(\frac 2 \delta\right)^{1/p'}.
    \]
\end{theorem}

\subsection{Compactification via Outer Brackets}

The outer bracket is defined as follows.

\begin{definition}
    An outer bracket for probability measure $\nu$ at scale $\epsilon$
    consists of two triples, one each for lower and upper bounds.
    \begin{enumerate}
        \item The function $\ell$, function class $Z_\ell$, and set $B_\ell$ satisfy
            two conditions: if $x \in B_\ell^c$ and $\phi\in Z_\ell$, then $\ell(x) \leq
            \phi(x)$,
            and secondly $|\int_{B_\ell^c} \ell(x)d\nu(x)| \leq \epsilon$.
        \item Similarly, function $u$, function class $Z_u$, and set $B_u$ satisfy:
            if $x\in B_u^c$ and $\phi\in Z_u$, then $u(x) \geq \phi(x)$,
            and secondly $|\int_{B_u^c} u(x)d\nu(x)| \leq \epsilon$.
    \end{enumerate}
\end{definition}

Direct from the definition, given bracketing functions $(\ell,u)$, a bracketed function
$\phi_f(\cdot;P)$, and the bracketing set $B:= B_u \cup B_\ell$,
\begin{equation}
    -\epsilon
    \leq \int_{B^c} \ell(x)d\nu(x)
    \leq \int_{B^c} \phi_f(x;P)d\nu(x)
    \leq \int_{B^c} u(x)d\nu(x)
    \leq \epsilon;
    \label{eq:outer_bracket}
\end{equation}
in other words, as intended, this mechanism allows deviations on $B^c$ to be discarded.
Thus to uniformly control the deviations of the dominated functions
$Z := Z_u \cup Z_\ell$ over the set $B^c$,
it suffices to simply control the deviations of the pair
$(\ell,u)$.

The following \namecref{fact:km:outer_bracket:2}
shows that a bracket exists for $\{\phi_f(\cdot; P) : P \in \Hkm_f(\nu;c,k)\}$
and compact $B$, and moreover that this allows sampled points and candidate centers
in far reaches to be deleted.

\begin{lemma}
    \label{fact:km:outer_bracket:2}
    Consider the setting and definitions in
    \Cref{fact:km:basic},
    but additionally define
  %%Let $\|\cdot\|$, $\rho$, $\Hkm(f,k)$, $c$, $\epsilon>0$ be as specified in
  %%\Cref{fact:km:basic}.  Define
  %%$p' := p/2-1$ and $M' := 2^{p'}\epsilon$, and
    \begin{align*}
     %% R &:=
     %%     \max\left\{
     %%         (2M)^{1/p}
     %%         + \sqrt{4c/r_1}
     %%         ,
     %%         \max_{i\in [p']}
     %%         (M/\epsilon)^{1/(p-2i)}
     %% \right\}
     %% \\
     %% B &:= \left\{
     %%     x\in\R^d
     %%     :
     %%     \|x - \bbE_\rho(X)\|
     %%     < R
     %% \right\}
     %% \\
        M' := 2^{p'}\epsilon,
        \qquad
        \ell(x)
        %&
        := 0,\qquad
        %\\
        u(x)
        %&
        := 4r_2\|x - \bbE(X)\|^2,
        \qquad
        %\\
        \epsilon_{\hat\rho}
        &:=
        \epsilon
        + \sqrt{\frac{M'ep'}{2m}} \left(\frac 2 \delta\right)^{1/p'}.
     %% \\
     %% C &:= \left\{
     %%     x\in\R^d
     %%     :
     %%     \|x - \bbE_\rho(X)\|
     %%     <
     %%     1+2\left((2M)^{1/p} + \sqrt{4c/r_1}\right)
     %%     +
     %%     R
     %%     \right\}.
    \end{align*}
    The following statements hold with probability at least $1-2\delta$
    over a draw of size
    $m \geq \max\{p'/(M'e), 9\ln(1/\delta)\}$.
    %moved conditioning up to properly handle c in \Hkm_f(\hat\rho;c,k)
  %%(Indeed, both statements pertaining to
  %%$\hat\rho$ hold simultaneously with probability at least $1-2\delta$.)
  %%\red{(need some probabilities due to the basic controls)}
    \begin{enumerate}
        \item
            $(u,\ell)$ is an outer bracket for $\rho$ at scale $\epsilon_\rho := \epsilon$
            with sets $B_\ell = B_u=B$
            and $Z_\ell = Z_u = \{\phi_f(\cdot; P) : P \in \Hkm_f(\hat\rho;c,k) \cup
            \Hkm_f(\rho;c,k)\}$, %\Hkm(f,k;c)$,
           %and with probability at least $1-2\delta$ over a draw of size
           %$m \geq p'/(M'e)$,
            and furthermore
            the pair
            $(u,\ell)$
            is also an outer bracket for $\hat\rho$ at scale $\epsilon_{\hat\rho}$
            with the same sets.
        \item
            %For every $\phi(\cdot;P)\in \Hkm(f,k;c)$,
            For every $P \in \Hkm_f(\hat\rho;c,k) \cup \Hkm_f(\rho;c,k)$,
            \[
                \left|
                \int \phi_f(x;P) d\rho(x)
                -
                \int_B \phi_f(x;P\cap C) d\rho(x)
                \right|
                \leq \epsilon_\rho = \epsilon.
            \]
            and
           %with probability at least $1-2\delta$ over a draw of size
           %$m \geq p'/(M'e)$,
            \[
                \left|
                \int \phi_f(x;P) d\hat\rho(x)
                -
                \int_B \phi_f(x;P\cap C) d\hat\rho(x)
                \right|
                \leq \epsilon_{\hat\rho}
                .
            \]
    \end{enumerate}
\end{lemma}

The proof of \Cref{fact:km:outer_bracket:2} has roughly the following outline.

\begin{enumerate}
    \item Pick some ball $B_0$ which has probability mass at least $1/4$.
        It is not possible for
        an element of $\Hkm_f(\hat\rho;c,k)\cup \Hkm_f(\rho;c,k)$ to have all centers
        far from $B_0$, since otherwise
        the cost is larger than $c$.  (Concretely, ``far from'' means
        at least $\sqrt{4c/r_1}$ away; note that this term appears in the
        definitions of $B$ and $C$ in
        \Cref{fact:km:basic}.)
        Consequently, at least one center lies near to $B_0$;
        this reasoning was also the first step in the $k$-means consistency proof due to
        $k$-means \citet{pollard_km_cons}.
    \item
        It is now easy to dominate
        $P \in \Hkm_f(\hat\rho;c,k)\cup \Hkm_f(\rho;c,k)$
        far away from $B_0$.
        In particular, choose any $p_0 \in B_0 \cap P$, which was
        guaranteed to exist in the preceding point; since
        $\min_{p\in P} \sfB_f(x,p) \leq \sfB_f(x,p_0)$ holds for all $x$, it suffices
        to dominate $p_0$.  This domination proceeds exactly as discussed in the
        introduction; in fact, the factor 4 appeared there, and again appears in
        the $u$ here, for exactly the same reason.
        Once again, similar reasoning can be found in the proof by
        \citet{pollard_km_cons}.
      %%It is now easy to dominate $\phi(\cdot;P) \in \Hkm(f,k;c)$ far away from $B_0$;
      %%pick some $p'\in P$ close to $B_0$, and note
      %%$\min_{p\in P} \|x-p\|^2 \leq \|p' - x\|^2 \leq 2\|p'\|^2 + 2\|x\|^2$.
      %%If $B^c$ is far enough from $B_0$ that $\|x\|$ dominates $\|p'\|$, then
      %%the existence of a dominating function $u$ follows.  Here again, similar
      %%reasoning can be found in the proof by \citet{pollard_km_cons}.
    \item
        Satisfying the integral conditions over $\rho$ is easy:
        it suffices to make $B$ huge.
        To control the size of $B_0$, as well as the size of $B$, and moreover
        the deviations of the bracket over $B$, the moment tools
        from \Cref{sec:moments} are used.
\end{enumerate}
Now turning consideration back to the proof of \Cref{fact:km:basic}, the above
bracketing allows the removal of points and centers outside of a compact set (in
particular, the pair of compact sets $B$ and $C$, respectively).
%XXX lelz following is false, 2\epsilon also from cover
%%This process introduces the standalone term $\epsilon$ present in the bound.
On the remaining truncated data and set of centers,
any standard tool suffices; for mathematical convenience, and also
to fit well with the use of norms in the definition of moments as well as the
conditions on the convex function $f$ providing the divergence $\sfB_f$,
norm structure used throughout the other properties,
covering arguments are used here.  (For details, please see \Cref{sec:km:deferred}.)

\section{Interlude: Refined Estimates via Clamping}
\label{sec:km:clamp}

So far, rates have been given that guarantee uniform convergence when the distribution
has a few moments, and these rates improve with the availability of higher moments.
These moment conditions, however, do not necessarily reflect any natural cluster
structure in the source distribution.  The purpose of this section is to propose
and analyze another distributional property which is intended to capture
cluster structure.
To this end, consider the following definition.

\begin{definition}
    Real number $R$ and compact set $C$ are a \emph{clamp} for
    probability measure $\nu$ and family of centers $Z$ and cost $\phi_f$
    at scale $\epsilon>0$ if every $P \in Z$
    satisfies
    \[
        \left|
            \bbE_\nu(\phi_f(X;P))
            -
            \bbE_\nu\left(\min\left\{\phi_f(X;P\cap C)\ , \ R\right\}\right)
        \right|
        \leq \epsilon.
    \]
\end{definition}

Note that this definition is similar to the second part of the outer bracket
guarantee in \Cref{fact:km:outer_bracket:2}, and, predictably enough, will soon
lead to another deviation bound.

\begin{example}
    If the distribution has bounded support, then choosing a clamping value $R$ and
    clamping set $C$ respectively slightly larger than the support size and set is
    sufficient: as was reasoned in the construction of outer brackets,
    if no centers are close to the support, then the cost is bad.  Correspondingly,
    the clamped set of functions $Z$ should again be choices of centers
    whose cost is not too high.
  %%\red{this is kindof stupid.  just say it in words.  being half rigorous (leaving out finite sample effects) is useless. also this part is supposed to be an interlude}
  %%Suppose $\rho$ is supported on a ball $S$ of radius $R_S$ centered at the origin,
  %%and let $S'$ denote the ball of radius $2R_S$.
  %%Each $\phi(\cdot;P) = \Hkm(f,k;r_1D^2)$ always has $P\cap S' \neq \emptyset$,
  %%since
  %%\[
  %%    \inf_{p \in S'} \inf_{x\in S} \sfB_f(x,p)
  %%    \geq r_1 \inf_{p \in S'} \inf_{x\in S} \|x-p\|^2
  %%    \geq r_1 D^2;
  %%\]
  %%As such, taking the clamping value $3R_S$ and a clamping set of radius $5R_S$
  %%is guaranteed to not change the cost.

    For a more interesting example, suppose $\rho$ is supported on $k$
    small balls of radius $R_1$, where the distance between their respective centers
    is some $R_2\gg R_1$.  Then by reasoning similar to the bounded case,
    all choices of
    centers achieving a good cost will place centers near to each ball, and thus
    the clamping value can be taken closer to $R_1$.
%
%%%%\red{is this too vague to be useful?}
\end{example}

Of course, the above gave the existence of clamps under favorable conditions.
The following shows that outer brackets can be used to show the existence of clamps
in general. In fact, the proof is very short, and follows the scheme laid out in the
bounded example above: outer bracketing allows the restriction of consideration to a
bounded set, and some algebra from there gives a conservative upper bound for the clamping
value.

\begin{proposition}
    \label{fact:bracket_gives_clamp}
    Suppose the setting and definitions of \Cref{fact:km:outer_bracket:2},
    and additionally define
    \[
        R := 2((2M)^{2/p} + R_B^2).
    \]
    Then $(C,R)$ is a clamp for measure $\rho$ and center $\Hkm_f(\rho;c,k)$
    at scale $\epsilon$, and with probability at least $1-3\delta$ over a draw
    of size $m\geq \max\{p'/(M'e),9\ln(1/\delta)\}$, it is also a clamp for $\hat\rho$
    and centers $\Hkm_f(\hat\rho;c,k)$ at scale $\epsilon_{\hat\rho}$.
\end{proposition}

The general guarantee using clamps is as follows.  The proof is almost the same
as for \Cref{fact:km:basic}, but note that this statement is not used quite as
readily,
since it first requires the construction of clamps.
%%  (there isn't much point in using the clamp provided
%%  above in \Cref{fact:bracket_gives_clamp}).

\begin{theorem}
    \label{fact:km:clamp:basic}
    Fix a norm $\|\cdot\|$.
    Let $(R,C)$ be a clamp for probability measure $\rho$ and empirical
    counterpart $\hat\rho$ over some center class $Z$ and cost $\phi_f$
    at respective scales $\epsilon_\rho$ and $\epsilon_{\hat\rho}$,
    where $f$ has corresponding convexity constants $r_1$ and $r_2$.
    Suppose $C$ is contained within a ball of radius $R_C$,
    let $\epsilon>0$ be given, define scale parameter
    \[
        \tau := \min\left\{
            \sqrt{\frac {\epsilon}{2r_2}}
            ,
            \frac {r_1\epsilon}{2r_2R_3}
        \right\},
    \]
    and let $\cN$ be a cover of $C$ by $\|\cdot\|$-balls of radius $\tau$
    (as per \cref{fact:cover:lp_balls}, if $\|\cdot\|$ is
    an $l_p$ norm, then $|\cN|\leq (1 + (2R_Cd)/ \tau)^d$ suffices).
    Then
    with probability at least $1-\delta$ over the draw of a sample of size
    $m \geq p'/(M'e)$,
    every set of centers $P\in Z$ satisfies
    \[
        \left|
        \int \phi_f(x;P)d\rho(x)
        - \int \phi_f(x;P)d\hat\rho(x)
        \right|
       %\leq 4\epsilon
       %+
       %4r_2R_C^2
       %\sqrt{
       %    \frac {1}{2m}
       %    \ln
       %    \left(
       %        \frac {2|\cN|^k}{\delta}
       %    \right)
       %}
       %+ \sqrt{\frac{M'ep'}{2m}} \left(\frac 2 \delta\right)^{1/p'}.
        \leq 2\epsilon +  \epsilon_\rho + \epsilon_{\hat\rho}
        +
        R^2
        \sqrt{
            \frac {1}{2m}
            \ln
            \left(
                \frac {2|\cN|^k}{\delta}
            \right)
        }.
    \]
\end{theorem}

Before adjourning this section, note that clamps and outer brackets disagree
on the treatment of the outer regions:
the former replaces the cost there with the fixed value $R$,
whereas the latter uses the value 0.
On the technical side, this
is necessitated by the covering argument used to produce the final theorem: if the clamping
operation instead truncated beyond a ball of radius $R$ centered at each $p\in P$, then
the deviations would be wild as these balls moved and suddenly switched the value at a
point from 0 to something large.  This is not a problem with
outer bracketing,
since the same points (namely $B^c$) are ignored by every set of centers.

\section{Mixtures of Gaussians}
\label{sec:mog}

Before turning to the deviation bound,
it is a good place to discuss the condition
$\sigma_1 I \preceq \varSigma \preceq \sigma_2 I$, which must be met by every covariance
matrix of every constituent Gaussian in a mixture.

The lower bound $\sigma_1 I \preceq \varSigma$, as discussed previously, is fairly common
in practice, arising either via a Bayesian prior, or by implementing
EM with an explicit condition that covariance updates are discarded when the eigenvalues
fall below some threshold.
In the analysis here, this lower bound is used to rule out two kinds of bad behavior.
\begin{enumerate}
    \item
        Given a budget of at least 2 Gaussians, and a sample of at least 2 distinct points,
        arbitrarily large likelihood may be achieved by devoting one Gaussian to one point,
        and shrinking its covariance.  This issue destroys convergence properties of
        maximum likelihood, since the likelihood score
        may be arbitrarily large over every sample,
        but is finite for well-behaved distributions.
        The condition $\sigma_1 I \preceq \varSigma$ rules this out.
    \item
        Another phenomenon is a ``flat'' Gaussian, meaning a Gaussian whose density
        is high along a lower dimensional manifold, but small elsewhere.
        Concretely,
        consider a Gaussian over $\R^2$ with covariance
        $\varSigma = \textup{diag}(\sigma, \sigma^{-1})$;
        as $\sigma$ decreases,
        the Gaussian has large density on a line, but low density elsewhere.
        This phenomenon is distinct from the preceding in that it does not
        produce arbitrarily large likelihood scores over finite samples.
        The condition $\sigma_1 I \preceq \varSigma$ rules this situation out as well.
    \end{enumerate}
    In both the hard and soft clustering analyses here, a crucial early step
    allows the assertion that good scores in some region mean the relevant parameter
    is nearby.  For the case of Gaussians, the condition $\sigma_1 I \preceq \varSigma$
    makes this problem manageable, but there is still the possibility that some far away,
    fairly uniform Gaussian has reasonable density.  This case is ruled out
    here via $\sigma_2 I \succeq \varSigma$.

\begin{theorem}
    \label{fact:mog:basic}
    Let probability measure $\rho$ be given with order-$p$ moment bound
    $M$ according to norm $\|\cdot\|_2$ where $p\geq 8$ is
    a positive multiple of 4,
    covariance bounds $0 < \sigma_1 \leq \sigma_2$ with $\sigma_1 \leq 1$ for simplicity,
   %and corresponding class $\Smog(\sigma_1,\sigma_2,k)$,
    and real $c\leq 1/2$ be given.
    Then with probability at least $1-5\delta$
    over the draw of a sample of size
    $
        m \geq \max\left\{
            (p / (2^{p/4+2}e))^2,
            8\ln(1/\delta),
            d^2\ln(\pi\sigma_2)^2\ln(1/\delta)
        \right\},
        $
    every set of Gaussian mixture
    parameters $(\alpha,\Theta) \in
    \Smog(\hat \rho ; c,k, \sigma_1,\sigma_2)
    \cup \Smog(\rho ; c,k, \sigma_1,\sigma_2)$
    satisfies
    \begin{align*}
        &\left|
        \int \phig(x;(\alpha,\Theta))d\rho(x)
        - \int \phig(x;(\alpha,\Theta))d\hat\rho(x)
        \right|
        \\
        &\qquad\qquad=
        \cO\left(m^{-1/2+3/p}
        \left(
            1+ \sqrt{\ln(m) + \ln(1/\delta)}
            + (1/\delta)^{4/p}
    \right)\right),
    \end{align*}
       where the $\cO(\cdot)$ drops numerical constants, polynomial terms depending
       on $c$, $M$, $d$, and $k$, $\sigma_2/\sigma_1$, and $\ln(\sigma_2/\sigma_1)$,
           but in particular has no sample-dependent quantities.
\end{theorem}

The proof follows the scheme of the hard clustering analysis.
One distinction is that the outer bracket now uses both components; the upper component
is the log of the largest possible density --- indeed, it is $\ln((2\pi\sigma_1)^{-d/2})$
--- whereas the lower component is a function mimicking the log density
of the steepest possible Gaussian --- concretely, the lower bracket's definition
contains the expression
$\ln((2\pi\sigma_2)^{-d/2}) - 2\|x-\bbE_\rho(X)\|_2^2/\sigma_1$, which lacks the
normalization of a proper Gaussian, highlighting the fact that bracketing elements
need not be elements of the class.
Superficially, a second distinction with the hard clustering
case is that far away Gaussians
can not be entirely ignored on local regions; the influence is limited, however,
and the analysis proceeds similarly in each case.

%   Since the proof follows the scheme of the $k$-means analysis, it is not sketched here,
%   but a few essential points are as follows.

%   \begin{enumerate}
%       \item Both parts of the outer bracket are useful; in particular,
%           they are defined as \red{meh say it in words}
%           \begin{align*}
%               u(x) &:= \ln(p_{\max})= -\frac {d}{2} \ln(2\pi \sigma),
%               \\
%               \ell(x) &:=
%               \ln(p_0/p_{\max}) - \frac d 2 \ln(2\pi\sigma_3)
%               - \frac {2}{\sigma}\|x-\bbE_\rho(X)\|_2^2.
%           \end{align*}
%           The upper bound comes for free by the condition $\varSigma \succeq \sigma I$;
%           indeed, for any Gaussian with parameters $\theta = (\mu,\varSigma)$ satisfying
%           $\varSigma \succeq \sigma I$,
%           \[
%               p_\theta(x) \leq \frac 1  {(2\pi\sigma)^{d/2}} \exp(\ 0\ ) = \exp(u(x)).
%           \]
%           On the other hand $\ell$ has quadratic term resembling the steepest
%           possible Gaussian, and with some offsets for convenience elsewhere.
%       \item
%           The hard clustering nature of the $k$-means proof meant the effects of far
%           away clusters really can be ignored; contrastingly, the mixture here can never
%           completely ignore far away centers or data, which adds many more terms to keep
%           track of and complicates some arguments (which still follow the same skeleton.
%   \end{enumerate}

\subsubsection*{Acknowledgments}
The authors thank the NSF for supporting this work under grant IIS-1162581.

\clearpage

%\subsubsection*{References}

\addcontentsline{toc}{section}{References}
\bibliographystyle{unsrtnat} %stipulated in nips docs..
\bibliography{km}

\begin{thebibliography}{21}
\providecommand{\natexlab}[1]{#1}
\providecommand{\url}[1]{\texttt{#1}}
\expandafter\ifx\csname urlstyle\endcsname\relax
  \providecommand{\doi}[1]{doi: #1}\else
  \providecommand{\doi}{doi: \begingroup \urlstyle{rm}\Url}\fi

\bibitem[Pollard(1981)]{pollard_km_cons}
David Pollard.
\newblock Strong consistency of k-means clustering.
\newblock \emph{The Annals of Statistics}, 9\penalty0 (1):\penalty0 135--140,
  1981.

\bibitem[Lugosi and Zeger(1994)]{Lugosi94ratesof}
Gbor Lugosi and Kenneth Zeger.
\newblock Rates of convergence in the source coding theorem, in empirical
  quantizer design, and in universal lossy source coding.
\newblock \emph{IEEE Trans. Inform. Theory}, 40:\penalty0 1728--1740, 1994.

\bibitem[Ben-david(2004)]{Ben-david04aframework}
Shai Ben-david.
\newblock A framework for statistical clustering with a constant time
  approximation algorithms for k-median clustering.
\newblock In \emph{COLT}, pages 415--426. Springer, 2004.

\bibitem[Rakhlin and Caponnetto(2006)]{DBLP:conf/nips/RakhlinC06}
Alexander Rakhlin and Andrea Caponnetto.
\newblock Stability of $k$-means clustering.
\newblock In \emph{NIPS}, pages 1121--1128, 2006.

\bibitem[Duda et~al.(2001)Duda, Hart, and Stork]{duda_hart_stork}
Richard~O. Duda, Peter~E. Hart, and David~G. Stork.
\newblock \emph{Pattern Classification}.
\newblock Wiley, 2 edition, 2001.

\bibitem[Ferguson(1996)]{ferguson_large_sample_theory}
Thomas~S. Ferguson.
\newblock \emph{A course in large sample theory}.
\newblock Chapman \& Hall, 1996.

\bibitem[Pollard(1982)]{pollard_km_clt}
David Pollard.
\newblock A central limit theorem for k-means clustering.
\newblock \emph{The Annals of Probability}, 10\penalty0 (4):\penalty0 919--926,
  1982.

\bibitem[Ben-david et~al.(2006)Ben-david, Luxburg, and
  P\'al]{Ben-david06asober}
Shai Ben-david, Ulrike~Von Luxburg, and D\'avid P\'al.
\newblock A sober look at clustering stability.
\newblock In \emph{In COLT}, pages 5--19. Springer, 2006.

\bibitem[Shamir and Tishby(1982)]{Shamir82clusterstability}
Ohad Shamir and Naftali Tishby.
\newblock Cluster stability for finite samples.
\newblock In \emph{Annals of Probability, 10(4)}, pages 919--926, 1982.

\bibitem[Shamir and Tishby(2008)]{Shamir_modelselection}
Ohad Shamir and Naftali Tishby.
\newblock Model selection and stability in k-means clustering.
\newblock In \emph{COLT}, 2008.

\bibitem[Boucheron et~al.(2005)Boucheron, Bousquet, and Lugosi]{bbl_esaim}
St{\'e}phane Boucheron, Olivier Bousquet, and G{\'a}bor Lugosi.
\newblock Theory of classification: a survey of recent advances.
\newblock \emph{ESAIM: Probability and Statistics}, 9:\penalty0 323--375, 2005.

\bibitem[Boucheron et~al.(2013)Boucheron, Lugosi, and Massart]{blm_conc}
St\'ephane Boucheron, G\'abor Lugosi, and Pascal Massart.
\newblock \emph{Concentration Inequalities: A Nonasymptotic Theory of
  Independence}.
\newblock Oxford, 2013.

\bibitem[Wellner(2005)]{wellner_uw_overview}
Jon Wellner.
\newblock Consistency and rates of convergence for maximum likelihood
  estimators via empirical process theory.
\newblock 2005.

\bibitem[van~der Vaart and Wellner(1996)]{vdw_wellner}
Aad van~der Vaart and Jon Wellner.
\newblock \emph{Weak Convergence and Empirical Processes}.
\newblock Springer, 1996.

\bibitem[Gao and Wellner(2009)]{wellner_kmonotone_mle}
FuChang Gao and Jon~A. Wellner.
\newblock On the rate of convergence of the maximum likelihood estimator of a
  k-monotone density.
\newblock \emph{Science in China Series A: Mathematics}, 52\penalty0
  (7):\penalty0 1525--1538, 2009.

\bibitem[Censor and Zenios(1997)]{censor_zenios}
Yair~Al Censor and Stavros~A. Zenios.
\newblock \emph{Parallel Optimization: Theory, Algorithms and Applications}.
\newblock Oxford University Press, 1997.

\bibitem[Banerjee et~al.(2005)Banerjee, Merugu, Dhillon, and
  Ghosh]{bregman_clustering}
Arindam Banerjee, Srujana Merugu, Inderjit~S. Dhillon, and Joydeep Ghosh.
\newblock Clustering with {B}regman divergences.
\newblock \emph{Journal of Machine Learning Research}, 6:\penalty0 1705--1749,
  2005.

\bibitem[Tao(2010)]{tao_conc_notes}
Terence Tao.
\newblock 254a notes 1: Concentration of measure, January 2010.
\newblock URL
  \url{http://terrytao.wordpress.com/2010/01/03/254a-notes-1-concentration-of-measure/}.

\bibitem[Pinelis and Utev(1984)]{pinelis_utev_rosenthal}
I.~F. Pinelis and S.~A. Utev.
\newblock Estimates of the moments of sums of independent random variables.
\newblock \emph{Teor. Veroyatnost. i Primenen.}, 29\penalty0 (3):\penalty0
  554--557, 1984.
\newblock Translation to English by Bernard Seckler.

\bibitem[Shalev-Shwartz(2007)]{ShaiThesis}
Shai Shalev-Shwartz.
\newblock \emph{Online Learning: Theory, Algorithms, and Applications}.
\newblock PhD thesis, The Hebrew University of Jerusalem, July 2007.

\bibitem[Hiriart-Urruty and Lemar\'echal(2001)]{HULL}
Jean-Baptiste Hiriart-Urruty and Claude Lemar\'echal.
\newblock \emph{Fundamentals of Convex Analysis}.
\newblock Springer Publishing Company, Incorporated, 2001.

\end{thebibliography}

\clearpage

\appendix

\section{Moment Bounds}
\label{sec:moments}

This section provides the basic probability controls resulting from moments.
The material deals with the following slight generalization of the bounded moment
definition from \Cref{sec:setup}.

\begin{definition}
    A function $\tau:\cX\to \R^d$ has \emph{order-$p$ moment bound $M$ for
    probability measure $\rho$ with respect to norm $\|\cdot\|$}
    if $\bbE_\rho(\|\tau(X)\|^l) \leq M$ for all $1\leq l\leq p$.
    (For convenience, measure $\rho$ and norm $\|\cdot\|$ will be often be implicit.)
\end{definition}

To connect this to the earlier definition, simply choose the map
$\tau(x) := x - \bbE_\rho(X)$.
As was the case in \Cref{sec:setup}, this definition requires a uniform bound
across all $l^{\textup{th}}$ moments for $1\leq l \leq p$.  Of course,
working with a probability measure implies these moments are all finite
when just the $p^{\textup{th}}$ moment is finite.  The significance of working
with a bound across all moments will be discussed again
in the context of \Cref{fact:chebyshev_tao} below.

The first result controls the measures of balls thanks to moments.
This result is only stated for the source distribution $\rho$,
but Hoeffding's inequality suffices to control $\hat\rho$.

\begin{lemma}
    \label{fact:moments:ball}
    Suppose $\tau$ has order-$p$ moment bound $M$.  Then for any $\epsilon > 0$,
  %%\red{XXX changed to $\leq$ from $<$ in the statement below.  need to check all
  %%applications...}
    % I seem to have checked them ...
    \[
        \Pr\left[
            \|\tau(X)\| \leq (M/\epsilon)^{1/p}
        \right]
        \geq 1 - \epsilon.
    \]
\end{lemma}
\begin{proof}
    If $M=0$, the result is immediate.  Otherwise, when $M>0$,
    for any $R > 0$, by Chebyshev's inequality,
    \[
        \Pr\left[
            \|\tau(X)\| < R
        \right]
        =
        1-
        \Pr\left[
            \|\tau(X)\| \geq R
        \right]
        \geq
        1-
        \frac {\bbE \|\tau(X)\|^p}{R^p}
        \geq
        1-
        \frac {M}{R^p};
    \]
    the result follows by choosing $R := (M/\epsilon)^{1/p}$.
\end{proof}

The following fact will be the basic tool for controlling empirical averages via moments.
Both the statement and proof are close to one by \citet[Equation 7]{tao_conc_notes},
which rather than bounded moments uses boundedness (almost surely).
As discussed previously,
the term $1/\delta^{1/l}$ overtakes $\ln(1/\delta)$
when
$l = \ln(1/\delta) / \ln(\ln(1/\delta))$.

For simplicity, this result is stated in terms of univariate random variables; to connect
with the earlier development, the random variable $X$ will be substituted with
the map
$x\mapsto \|\tau(x)\|$.

%XXX there's also the exercise from lugosi book which says the best moment bound is
%better than the best chernoff bound.  .. but I don't have the book with me..
%suppose I can look it up from an email to sanjoy.

\begin{lemma}(Cf. \citet[Equation 7]{tao_conc_notes}.)
    \label{fact:chebyshev_tao}
    Let $m$ i.i.d. copies $\{X_i\}_{i=1}^m$ of a random variable $X$,
    even integer $p\geq 2$,
    real $M> 0$ with $\bbE(|X-\bbE(X)|^l) \leq M$ for $2 \leq l \leq p$,
    and $\epsilon > 0$ be given.
    %If $n \geq \frac {p}{Me}$, then
    %If $m \geq p/e$, then
    If $m \geq p/(Me)$, then
    \[
        \Pr\left(\left|
            \frac 1 n \sum_{i} X_i - \bbE(X)
            \right| \geq \epsilon
        \right) \leq
        \frac {2}{(\epsilon\sqrt{m})^p}
        \left(\frac{Mpe}{2}\right)^{p/2}.
    \]
    In other words, with probability at least $1-\delta$ over a draw of size
    $m\geq p/(Me)$,
    \[
        \left|
        \frac 1 n \sum_{i} X_i - \bbE(X)
        \right|
        \leq
        \sqrt{\frac{Mpe}{2m}}\left(\frac 2 \delta \right)^{1/p}.
    \]
\end{lemma}
\begin{proof}
    Without loss of generality, suppose $\bbE(X_1) = 0$ (i.e., given $Y_1$ with
    $\bbE(Y_1) \neq 0$, work with $X_i := Y_i - \bbE(Y_1)$).
    By Chebyshev's inequality,
    \begin{align}
        \Pr\left(\left|
            \frac 1 m \sum_{i} X_i
            \right| \geq \epsilon
        \right)
        &\leq \frac {\bbE\left|\frac 1 m \sum_i X_i\right|^p}{\epsilon^p}
        = \frac {\bbE\left|\sum_i X_i\right|^p}{(m\epsilon)^p}.
        \label{eq:cheby_many_2:0}
    \end{align}
    Recalling $p$ is even,
    consider the term
    \[
        \bbE\left|\sum_i X_i\right|^p
        = \bbE\left(\sum_i X_i\right)^p
        = \sum_{i_1,i_2,\ldots,i_p \in [m]} \bbE\left(
            \prod_{j=1}^p X_{i_j}
        \right).
    \]
    If some $i_j$ is equal to none of the others, then, by independence,
    a term $\bbE(X_{i_j}) = 0$ is introduced and the product vanishes; thus the
    product is nonzero when each $i_j$ has some copy $i_j = i_{j'}$, and thus there are
    at most $p/2$ distinct values amongst $\{i_j\}_{j=1}^p$.
    Each distinct value contributes a term $\bbE(X^l)\leq \bbE(|X|^l) \leq M$
    for some $2 \leq l \leq p$,
    and thus
    \begin{equation}
        \bbE\left|\sum_i X_i\right|^p
        %%bug in earlier version?:
        %%%no, I'm retarded, it was correct as originally written..
        %%%glad I was shaky about this late night, sickness-incuced tweak
        \leq \sum_{r = 1}^{p/2} M^r N_r,
        %\leq \sum_{r = 1}^{p/2} M^{p/2 -(r-1)} N_r
        %\leq M^{p/2}\sum_{r = 1}^{p/2} N_r,
        \label{eq:cheby_many_2:1}
    \end{equation}
    where $N_r$ is the number of ways to choose a multiset of size $p$ from $[m]$,
    subject to the constraint that each number appears at least twice, and at most
    $r$ distinct numbers appear.  One way to over-count this is to first choose a
    subset of size $r$ from $m$, and then draw from it (with repetition) $p$ times:
    \[
        N_r
        \leq \binom{m}{r} r^p
        \leq \frac {m^rr^p}{r!}
        \leq \frac {m^rr^p}{(r/e)^r}
        = (me)^rr^{p-r}.
    \]
    Plugging this into \cref{eq:cheby_many_2:1}, and thereafter re-indexing with
    $r := p/2 - j$,
    \begin{align*}
        \bbE\left|\sum_i X_i\right|^p
        &\leq \sum_{r = 1}^{p/2} (Mme)^r r^{p-r}
        \leq \sum_{r = 1}^{p/2} (Mme)^r (p/2)^{p-r}
        \\
        %\leq M^{p/2}\sum_{r = 1}^{p/2} (me)^r r^{p-r}
        &\leq \sum_{j = 0}^{p/2} (Mme)^{p/2-j} (p/2)^{p/2+j}
        %= M^{p/2}\sum_{j = 0}^{p/2} (me)^{p/2-j} j^{p/2+j}
        \leq \left(\frac {Mmpe}{2}\right)^{p/2}
        %\leq \left(\frac {Mmpe}{2}\right)^{p/2}
        \sum_{j=0}^{p/2} \left(\frac {p}{2Mme}\right)^{j}.
        %\sum_{j=0}^{p/2} \left(\frac {p}{2me}\right)^{j}.
    \end{align*}
    Since $p \leq Mme$,
    %Since $p \leq me$,
    \[
        \bbE\left|\sum_i X_i\right|^p
        \leq 2 \left(\frac {Mmpe}{2}\right)^{p/2},
    \]
    and the result follows by plugging this into \cref{eq:cheby_many_2:0}.
\end{proof}

Thanks to Chebyshev's inequality, proving \Cref{fact:chebyshev_tao}
boils down to controlling $\bbE|\sum_i X_i|^p$,
which here relied on a combinatorial scheme by \citet[Equation 7]{tao_conc_notes}.
There is, however, another approach to controlling this quantity, namely Rosenthal
inequalities, which write this $p^{\textup{th}}$ moment of the sum
in terms of the $2^{\textup{nd}}$
and $p^{\textup{th}}$ moments of individual random variables
(general material on these bounds can be found in the book of
\citet[Section 15.4]{blm_conc},
however the specific form provided here is most easily presented by
\citet{pinelis_utev_rosenthal}).
While Rosenthal inequalities may seem a more elegant approach, they involve different
constants, and thus the approach and bound here are followed instead to suggest further
work on how to best control $\bbE|\sum_i X_i|^p$.

Returning to task,
as was stated in the introduction, the dominated convergence theorem provides
that $\int_{B_i} \|x\|_2^2d\rho(x) \to \int \|x\|_2^2d\rho(x)$ (assuming
integrability of $x\mapsto\|x\|_2^2$), where the sequence
of balls $\{B_i\}_{i=1}^\infty$ grow in radius without bound; moment bounds
allow the rate of this process to be quantified as follows.

%XXX maybe I should hav emade it a strict inequality all along.
\begin{lemma}
    \label{fact:moment_clipping}
    Suppose $\tau$ has order-$p$ moment bound $M$, and let $0 < k < p$ be given.
    Then for any $\epsilon > 0$,
    the ball
    \[
        B := \left\{ x \in \cX : \|\tau(X)\| \leq
            (M/\epsilon)^{1/(p-k)}
        \right\}
    \]
    satisfies
    \[
        \int_{B^c} \|\tau(x)\|^kd\rho(x) \leq \epsilon.
    \]
\end{lemma}
\begin{proof}
    Let the $B$ be given as specified;
    an application of \Cref{fact:moments:ball} with
    $\epsilon':= (\epsilon^p/M^k)^{1/(p-k)}$ yields
    \[
        \int \1[x \in B^c]d\rho(x)
        = \Pr[ \|\tau(x)\| > (M/\epsilon)^{1/(p-k)} ]
        = \Pr[ \|\tau(x)\| > (M/\epsilon')^{1/p} ]
        \leq \epsilon'.
    \]
    By H\"older's inequality
    with conjugate exponents $p/k$ and $p/(p-k)$ (where the condition $0< k < p$ means
    each lies within $(1,\infty)$),
    \begin{align*}
        \int_{B^c} \|\tau(x)\|^kd\rho(x)
        &= \int \|\tau(x)\|^k \1[x \in B^c]d\rho(x)
        \\
        &\leq
        \left(
            \int \|\tau(x)\|^{k(p/k)}d\rho(x)
        \right)^{k/p}
        \left(
            \int \1[x \in B^c]^{p/(p-k)}d\rho(x)
        \right)^{(p-k)/p}
        \\
        &\leq
        \left(
            M
        \right)^{k/p}
        \left(
            \frac {\epsilon^{p/(p-k)}}{M^{k/(p-k)}}
        \right)^{(p-k)/p}
        \\
        &=\epsilon
    \end{align*}
    as desired.
\end{proof}

Lastly, thanks to the moment-based deviation inequality in \Cref{fact:chebyshev_tao},
the deviations on this outer region may be controlled.  Note that in order to control
the $k$-means cost (i.e., an exponent $k=2$), at least 4 moments are necessary
($p\geq 4$).

%XXX \red{\textbf{need to check applicatoins of this.  since had to flip exponent.}}
%XXX hopefully okay now
\begin{lemma}
    \label{fact:moment_clipping:sample}
    Let integers $k\geq 1$ and $p'\geq 1$ be given, and set
    $\tilde p := k(p'+1)$.
    Suppose $\tau$ has order-$\tilde p$ moment bound $M$,
    and let $\epsilon > 0$ be arbitrary.
    Define the radius $R$ and ball $B$ as
    \[
        R := \max\{
            (M/\epsilon)^{1/(\tilde p-ik)}
            :
            1 \leq i < \tilde p/k
        \}
        \qquad\textup{and}\qquad
        B := \left\{
            x \in \cX : \|\tau(x)\|\leq R
        \right\},
    \]
    and set $M' := 2^{p'}\epsilon$.
    With probability at least $1-\delta$ over the draw of a sample
    of size $m \geq p' / (M'e)$,
    %of size $m \geq p/(ke)$,
    \[
        \left|
            \int_{B^c} \|\tau(x)\|^kd\hat\rho(x)
            -
            \int_{B^c} \|\tau(x)\|^kd\rho(x)
        \right|
        \leq
        \sqrt{
            \frac
            {M'ep'}
            {2m}
        }
        \left(
            \frac
            2\delta
        \right)^{1/p'}
        .
    \]
%
%
%   Suppose $\tau$ has order-$p$ moment bound $M$,
%   and let $1 \leq k < p$ be given where $k$ divides $p$,
%   and lastly let $\epsilon>0$ be arbitrary.
%   Define the radius $R$ and ball $B$ as
%   \[
%       R := \max\{
%           (M/\epsilon)^{1/(p-ik)}
%           :
%           1 \leq i < p/k
%       \}
%       \qquad\textup{and}\qquad
%       B := \left\{
%           x \in \cX : \|\tau(x)\|\leq R
%       \right\},
%   \]
%   and set $p' := p/k-1$ and $M' := 2^{p'}\epsilon$.
%   With probability at least $1-\delta$ over the draw of a sample
%   of size $m \geq p' / (M'e)$,
%   %of size $m \geq p/(ke)$,
%   \[
%       \left|
%           \int_{B^c} \|\tau(x)\|^kd\hat\rho(x)
%           -
%           \int_{B^c} \|\tau(x)\|^kd\rho(x)
%       \right|
%       \leq
%       \sqrt{
%           \frac
%           {M'ep'}
%           {2m}
%       }
%       \left(
%           \frac
%           2\delta
%       \right)^{1/p'}
%       .
%   \]
\end{lemma}

\begin{proof}
    Consider a fixed $1 \leq i < \tilde p / k = p'+1$,
    and set $l = ik$.
    Let $B_l$ be the ball provided by \Cref{fact:moment_clipping} for
    exponent $l$.
    Since $B\supseteq B_l$,
    \[
        \int_{B^c} \|\tau(x)\|^{l}d\rho(x)
        \leq \int_{B_l^c} \|\tau(x)\|^{l}d\rho(x)
        \leq \epsilon.
    \]
    As such, by Minkowski's inequality,
    since $z\mapsto z^l$ is convex for $l\geq 1$,
    \begin{align*}
        &\left(
            \int\left|
            \|\tau(x)\|\1[x\in B^c]
            -
            \int_{B^c}\|\tau(x)\|d\rho(x)
            \right|^ld\rho(x)
        \right)^{1/l}
        \\
        &\qquad\leq
        \left(
            \int_{B^c}
            \|\tau(x)\|^ld\rho(x)
        \right)^{1/l}
        +
        \left(\int_{B^c}\|\tau(x)\|d\rho(x)\right)^{l/l}
        \\
        &\qquad
        \leq
        2
        \left(
            \int_{B^c}
            \|\tau(x)\|^ld\rho(x)
        \right)^{1/l},
    \end{align*}
    meaning
    \begin{align*}
        \int\left|
        \|\tau(x)\|\1[x\in B^c]
        -
        \int_{B^c}\|\tau(x)\|d\rho(x)
        \right|^ld\rho(x)
        \leq
        2^l
        \int_{B^c}
        \|\tau(x)\|^l
        \leq
        2^l
        \int_{B_l^c}
        \|\tau(x)\|^l
        \leq
        2^l \epsilon.
    \end{align*}
    Since $l=ik$ had $1 \leq i < \tilde p/k=p'+1$ arbitrary, it
    follows that the map
    $x\mapsto \|\tau(x)\|^k\1[x\in B^c]$ has its first $p'$ moments
    bounded by $2^{p'}\epsilon$.

    The finite sample bounds will now proceed with an application of
    \Cref{fact:chebyshev_tao}, where the random variable
    $X$ will be the map $x\mapsto \|\tau(x)\|^k\1[x\in B^c]$.
    Plugging the above moment bounds for this random variable into
    \Cref{fact:chebyshev_tao}, the result follows.
  %%as such, taking
  %%$M' := 2^{p/k}\epsilon$ as in the statement,
  %%%by
  %%and an application of \Cref{fact:chebyshev_tao}
  %%gives the desired result.
  %%,
  %%with probability at least $1-\delta$,
  %%\[
  %%    \Pr\left(\left|
  %%        \int_{B^c} \|\tau(x)\|^kd\hat\rho(x)
  %%        -
  %%        \int_{B^c} \|\tau(x)\|^kd\rho(x)
  %%        \right| \geq \epsilon'
  %%    \right) \leq
  %%    \frac {2}{(\epsilon'\sqrt{m})^{p/k}}
  %%    \left(\frac{M'(p/k)e}{2}\right)^{p/(2k)}.
  %%\]
  %%Setting the right hand side equal to $\delta$ and solving for $\epsilon'$,
  %%the result follows.
\end{proof}

%XXX yeah oh well.
%   \begin{remark}
%       \red{i should check the rudelson etc stuff.}
%   \end{remark}

%XXX don't remember what this was supposed to be.. lelz
%   \begin{corollary}[Simplification for $k$-means and MoG]
%       \red{is this necessary..}
%   \end{corollary}

\section{Deferred Material from \Cref{sec:km}}
\label{sec:km:deferred}

Before proceeding with the main proofs,
note that Bregman divergences in the setting here are sandwiched between quadratics.

\begin{lemma}
    \label{fact:bregman:easy_norms}
    If differentiable $f$ is $r_1$ strongly convex with respect to $\|\cdot\|$,
    then $\sfB_f(x,y) \geq r_1 \|x-y\|^2$.
    If differentiable $f$ has Lipschitz gradients with parameter $r_2$
    with respect to $\|\cdot\|$,
    then $\sfB_f(x,y) \leq r_2 \|x-y\|^2$.
\end{lemma}
\begin{proof}
    The first part (strong convexity) is standard (see for instance the proof
    by \citet[Lemma 13]{ShaiThesis}, or a similar proof by
    \citet[Theorem B.4.1.4]{HULL}).
    For the second part,
    by the fundamental theorem of calculus,
    properties of norm duality, and the Lipschitz gradient property,
    \begin{align*}
        f(x)
        &= f(y) + \ip{\nf(y)}{x-y}
        + \int_0^1 \ip{\nf(y + t(x-y)) - \nf(y)}{x-y} dt
        \\
        &\leq f(y) + \ip{\nf(y)}{x-y}
        + \int_0^1 \|\nf(y + t(x-y)) - \nf(y)\|_*\|x-y\| dt
        \\
        &\leq f(y) + \ip{\nf(y)}{x-y}
        + \frac {r_2} 2 \|x-y\|^2.
    \end{align*}
    (The preceding is also standard; see for instance the beginning of a
    proof by \citet[Theorem E.4.2.2]{HULL}, which only differs by fixing
    the norm $\|\cdot\|_2$.)
\end{proof}

\subsection{Proof of \Cref{fact:km:outer_bracket:2}}
The first step is the following
characterization of $\Hkm_f(\nu;c,k)$: at least one center must fall within some
compact set.  (The lemma works more naturally with the contrapositive.)
The proof by \citet{pollard_km_cons} also started by controlling a single center.

\begin{lemma}
    \label{fact:km:one_center:compact}
    Consider the setting of \Cref{fact:km:outer_bracket:2}, and additionally
    define the two balls
    \begin{align*}
        B_0 &:= \left\{
            x\in \R^d
            :
            \|x - \bbE_\rho(X)\| \leq (2M)^{1/p}
        \right\},
        \\
        C_0 &:= \left\{
            x\in \R^d
            :
            \|x - \bbE_\rho(X)\| \leq (2M)^{1/p} + \sqrt{4c/r_1}
        \right\},
    \end{align*}
    Then $\rho(B_0) \geq 1/2$, and for any center set $P$,
    if $P\cap C_0 = \emptyset$ then $\bbE_\rho(\phi_f(X;P)) \geq 2c$.
    Furthermore, with probability at least $1-\delta$ over a draw from
    $\rho$ of size at least
    \[
        m \geq 9 \ln\left(\frac 1 \delta\right).
    \]
    then $\hat\rho(B_0) > 1/4$ and
    $P\cap C_0 = \emptyset$
    implies $\bbE_{\hat\rho}(\phi_f(X;P)) > c$.
\end{lemma}
\begin{proof}
    The guarantee $\rho(B_0)\geq 1/2$ is direct from \Cref{fact:moments:ball}
    with moment map $\tau(x) := x - \bbE_\rho(X)$.
    By Hoeffding's inequality and the lower bound on $m$,
    with probability at least $1-\delta$,
    \[
        \hat\rho(B_0)
        \geq \rho(B_0) - \sqrt{\frac 1 {2m} \ln\left(\frac 1 \delta\right)}
        > \frac 1 4.
    \]

    By the definition of $C_0$, every $p\in C_0^c$ and $x\in B_0$
    satisfies
    \[
        \sfB_f(x,p)
        \geq r_1 \|x-p\|^2
        \geq 4c.
    \]
    Now let $\nu$ denote either $\rho$ or $\hat\rho$; then for any set of centers
    $P$ with $P\cap C_0 = \emptyset$ (meaning $P\subseteq C_0^c$),
    \begin{align*}
        \int \phi_f(x;P) d\nu(x)
        &=
        \int \min_{p\in P} \sfB_f(x,p) d\nu(x)
        \\
        &\geq
        \int_{B_0} \min_{p\in P} \sfB_f(x,p) d\nu(x)
        \\
        &\geq
        \int_{B_0} \min_{p\in P} 4c d\nu(x)
        \\
        &=4c \nu(B_0).
    \end{align*}
    Instantiating $\nu$ with $\rho$ or $\hat \rho$, the results follow.
\end{proof}

With this tiny handle on the structure of a set of centers $P$
satisfying $\phi_f(x;P) \leq c$, the proof of \Cref{fact:km:outer_bracket:2} follows.

\begin{proof}[Proof of \Cref{fact:km:outer_bracket:2}]
    Throughout both sections, let
    $B_0$ and $C_0$ be as defined in \Cref{fact:km:one_center:compact};
    it follows by \Cref{fact:km:one_center:compact}, with probability at least
    $1-\delta$,
    that $P \in \Hkm_f(\rho;c,k) \cup \Hkm_f(\hat\rho;c,k)$
    implies $P\cap C_0\neq \emptyset$.  Henceforth discard this failure event,
    and fix any
    $P \in \Hkm_f(\rho;c,k) \cup \Hkm_f(\hat\rho;c,k)$.

    \begin{enumerate}
        \item
    Since $P\cap C_0\neq \emptyset$, fix some $p_0 \in P\cap C_0$.
    Since $B\supseteq C_0$ by definition, it follows, for every
    $x\in B^c$ that
    \begin{align*}
        \phi_f(x;P)
        &= \min_{p\in P} \sfB_f(x,p)
        \leq r_2\|x-p_0\|^2
        \leq r_2(\|x-\bbE_\rho(X)\| +\|p_0-\bbE_\rho(X)\|)^2
        \\
        &\leq 4r_2\|x-\bbE_\rho(X)\|^2
        = u(x).
    \end{align*}
    Additionally,
    \[
        \ell(x) = 0 \leq \min_{p\in P}r_1 \|x-p\|^2 \leq \phi_f(x;P),
    \]
    meaning $u$ and $\ell$ properly bracket $Z_\ell = Z_u$ over $B^c$;
    what remains is to control their mass over $B^c$.

    Since $\ell = 0$,
    \[
        \left|\int_{B^c} \ell(x)d\hat\rho(x)\right|
        = \left|\int_{B^c} \ell(x)d\rho(x)\right|
        = 0 < \epsilon.
    \]

    Next, for $u$ with respect to $\rho$, the result follows from
    the definition of $u$
    together with \Cref{fact:moment_clipping} (using the map
    $\tau(x) = x- \bbE_\rho(X)$ together with exponent 2).

    Lastly, to control $u$ with respect to $\hat\rho$,
    note that $p' \leq p/2-1$ means
    $\tilde p := 2(p' +1) \leq p$,
    and thus the map $\tau(x) := \|x-\bbE_\rho(X)\|^2$
    has order-$\tilde p$ moment bound $M$.
    Thus, by \Cref{fact:moment_clipping:sample} and
    the triangle inequality,
    \[
        \left|\int_{B^c} u(x) d\hat\rho(x)\right|
        \leq \epsilon
        + \sqrt{\frac{M'ep'}{2m}} \left(\frac 2 \delta\right)^{1/p'}
        = \epsilon_{\hat\rho}.
    \]
\item
    Throughout this proof, let $\nu$ denote either $\rho$ or $\hat\rho$;
    the above established
    \[
        \left|\int_{B^c} u(x)d\nu(x)\right| \leq \epsilon_\nu,
    \]
    where in the case of $\nu = \hat\rho$, this statement holds with probability
    $1-\delta$; henceforth discard this failure event, and thus the statement
    holds for both cases.

    %don't actually need following edge case?  seems i wrote things to avoid it..
   %If $P=P\cap C$, there is nothing to show,
   %thus suppose $P\setminus C\neq \emptyset$.
    By definition of $C$, for any $p\in C^c$
    and $x\in B$,
    \begin{align*}
        \sfB_f(x,p)
        &\geq r_1\|x-p\|^2
        \geq r_1\left(\sqrt{r_2/r_1}\left((2M)^{1/p} + \sqrt{4c/r_1}
            + R_B\right)\right)^2
        \\
        &=
        r_2\left((2M)^{1/p} + \sqrt{4c/r_1}
            + R_B\right)^2.
    \end{align*}
    On the other hand,
    fixing any $p_0\in P\cap C_0$ (which was guaranteed to
    exist at the start of this proof), since $C_0 \subseteq C$,
    \[
        \sup_{x\in B} \phi_f(x;P\cap C)
        \leq \sup_{x\in B} r_2 \|x - p_0\|^2
        %wat?
        %\leq (2M)^{1/p} + \sqrt{4c/r_1} + R_B.
        \leq r_2\left((2M)^{1/p} + \sqrt{4c/r_1} + R_B\right)^2.
    \]
    Consequently, no element of $B$ is closer to an element of $P\cap C$
    than to any element of $P\setminus C$.
    As such,
    \begin{align*}
        \int \phi_f(x;P)d\nu(x)
        \geq \int_B \phi_f(x;P)d\nu(x) + \int_{B^c} \ell(x) d\nu(x)
        = \int_B \phi_f(x;P\cap C)d\nu(x).
    \end{align*}
    (Note here that $\ell(x) = 0$ was used directly, rather than the $\epsilon$
    provided by outer covering; in the case of Gaussian mixtures,
    both bracket elements are nonzero,
    and $\epsilon$ will be used.)
    This establishes one direction of the bound.

    For the other direction, note that adding centers back in only decreases
    cost (because $\min_{p\in P\cap C}$ is replaced with $\min_{p \in P}$),
    and thus recalling the properties of the outer bracket element $u$ established
    above,
    \begin{align*}
        \int_B \phi_f(x;P\cap C)d\nu(x)
        &=
        \int \phi_f(x;P\cap C)d\nu(x)
        -
        \int_{B^c} \phi_f(x;P\cap C)d\nu(x)
        \\
        &\geq
        \int \phi_f(x;P\cap C)d\nu(x)
        -
        \int_{B^c} u(x)d\nu(x)
        \\
        &\geq
        \int \phi_f(x;P)d\nu(x)
        -
        \epsilon_\nu,
    \end{align*}
    which gives the result(s).
    \end{enumerate}
\end{proof}

\subsection{Covering Properties}

The next step is to control the deviations over the bounded portion; this is
achieved via uniform covers, as developed in this subsection.

First, another basic fact about Bregman divergences.

\begin{lemma}
    \label{fact:km:cover_prelim}
    Let differentiable convex function $f$ be given with Lipschitz gradient constant
    $r_2$ with respect to norm $\|\cdot\|$, and let $\sfB_f$ be the corresponding
    Bregman divergence.
    For any $\{x,y,z\}\subseteq \cX$,
    \[
        \sfB_f(x,z) \leq \sfB_f(x,y) +\sfB_f(y,z) + r_2\|x-y\|\|y-z\|.
    \]
    Similarly, given finite sets $Y\subseteq \cX$ and $Z\subseteq\cX$,
    and letting $Y(p)$ and $Z(p)$ respectively select (any) closest point
    in $Y$ and $Z$ to $p$ according to $\sfB_f$, meaning
    \[
        Y(p) := \argmin_{y\in Y} \sfB_f(y,p)
        \qquad\textup{and}\qquad
        Z(p) := \argmin_{z\in Z} \sfB_f(z,p),
    \]
    then
    \[
        \min_{z\in Z}\sfB_f(x,z)
        \leq \min_{y\in Y} \sfB_f(x,y)
        +\sfB_f(Y(x),Z(Y(x))) +
        r_2\|x-Y(x)\|\|Y(x) - Z(Y(x))\|.
    \]
\end{lemma}
\begin{proof}
    By definition of $\sfB_f$, properties of dual norms, and the Lipschitz gradient
    property,
    \begin{align*}
        \sfB_f(x,z) - \sfB_f(x,y) - \sfB(y,z)
        &= f(x) - f(z) - f(x) + f(y) -f(y) + f(z)
        \\
        &\qquad - \ip{\nf(z)}{x-z} + \ip{\nf(y)}{x-y} + \ip{\nf(z)}{y-z}
        \\
        &= \ip{\nf(y) - \nf(z)}{x-y}
        \\
        &\leq \|\nf(y) - \nf(z)\|_*\|x-y\|
        \\
        &\leq r_2\|y - z\|\|x-y\|;
    \end{align*}
    rearranging this inequality gives the first statement.

    The second statement follows the first instantiated with $y = Y(x)$ and $z = Z(Y(x))$,
    since
    \begin{align*}
        \min_{z\in Z} \sfB_f(x,z)
        &\leq \sfB_f(x,Z(Y(x)))
        \\
        &\leq \sfB_f(x,Y(x)) +\sfB_f(Y(x),Z(Y(x))) + r_2\|x-Y(x)\|\|Y(x) - Z(Y(x))\|,
    \end{align*}
    and using $\sfB_f(x,Y(x)) = \min_{y\in Y} \sfB_f(x,y)$.
\end{proof}

The covers will be based on norm balls; the following estimate is useful.

\begin{lemma}
    \label{fact:cover:lp_balls}
    If $\|\cdot\|$ is an $l_p$ norm over $\R^d$, then
    the ball of radius $R$ admits a cover $\cN$ with size
    \[
        |\cN|
        \leq
        \left(1 + \frac {2Rd}{\tau}\right)^d.
    \]
\end{lemma}
\begin{proof}
    It suffices to grid the $B$ with $l_\infty$ balls centered
    at grid points at scale $\tau/d$; the result follows since the $l_\infty$ balls
    of radius $\tau/d$ are contained in $l_p$ balls of radius $\tau$ for all $p\geq 1$.
\end{proof}

The uniform covering result is as follows.

\begin{lemma}
    \label{fact:bregman:cover}
    Let scale $\epsilon > 0$,
    ball $B := \{ x\in \R^d : \|x- \bbE(X)\| \leq R\}$,
    parameter set $Z := \{ x\in \R^d : \|x-\bbE(X)\| \leq R_2\}$,
    and differentiable convex function $f$ with Lipschitz gradient parameter $r_2$
    with respect to norm $\|\cdot\|$
    be given.
    Define resolution parameter
    \[
        \tau := \min\left\{
            \sqrt{\frac {\epsilon}{2r_2}}
            ,
            \frac {\epsilon}{2(R_2+R)r_2}
        \right\},
    \]
    and let $\cN$ be set of centers for a cover of $Z$ by $\|\cdot\|$-balls
    of radius $\tau$ (see \Cref{fact:cover:lp_balls} for an estimate when
    $\|\cdot\|$ is an $l_p$ norm).
    It follows that there exists a uniform cover $\cF$ at scale
    $\epsilon$ with cardinality
    $|\cN|^k$, meaning for any collection $P = \{p_i\}_{i=1}^l$
    with $p_i \in Z$ and $l\leq k$, there is a cover element $Q$
    with
    \[
        \sup_{x\in B}
        \left|
        \min_{p\in P}\sfB_f(x,p) - \min_{q \in Q}\sfB_f(x,q)
        \right| \leq \epsilon.
    \]
\end{lemma}
\begin{proof}
    Given a collection $P$ as specified, choose $Q$ so that for every
    $p\in P$, there is $q\in Q$ with $\|p-q\| \leq \tau$, and vice versa.
    By \Cref{fact:km:cover_prelim} (and using the notation therein),
    for any $x\in B^c$,
    \begin{align*}
        \min_{p\in P} \sfB_f(x,p)
        &\leq \min_{q\in Q} \sfB_f(x,q) + \sfB_f(Q(x), P(Q(x)))
        + r_2\|x - Q(x)\|\|Q(x) - P(Q(x))\|
        \\
        &\leq \min_{q\in Q} \sfB_f(x,q)
        + r_2 \tau^2
        + r_2 \tau (R+ R_2)
        \\
        &\leq \min_{q\in Q} \sfB_f(x,q)
        +\epsilon;
    \end{align*}
    the reverse inequality holds for the same reason, and the result follows.
\end{proof}

\subsection{Proof of \Cref{fact:km:basic} and \Cref{fact:km:basic:kmeans_cost}}

First, the proof of the general rate for $\Hkm_f(\nu;c,k)$.

\begin{proof}[Proof of \Cref{fact:km:basic}]
    For convenience, define $M' = 2^{p'}\epsilon$.
    By \Cref{fact:bregman:cover}, let $\cN$ be a cover of the set $C$,
    whereby every set of centers $P\subseteq C$ with $|P|\leq k$ has
    a cover element $Q\in \cN^k$ with
    \begin{equation}
        \sup_{x\in B}
        \left|
        \min_{p\in P}\sfB_f(x,p) - \min_{q \in Q}\sfB_f(x,q)
        \right| \leq \epsilon;
        \label{eq:km:basic:1}
    \end{equation}
    when $\|\cdot\|$ is an $l_p$ norm,
    \Cref{fact:cover:lp_balls} provides the stated estimate of its size.
    Since $B\subseteq C$ and
    \[
        \sup_{x\in B} \sup_{p\in C} \sfB_f(x,p)
        \leq r_2 \sup_{x\in B} \sup_{p\in C} \|x-p\|^2
        \leq 4r_2R_C^2,
    \]
    it follows by Hoeffding's inequality and a union bound over $\cN^k$
    that with probability at
    least $1-\delta$,
    \begin{equation}
        \sup_{Q\in\cN^k}
        \left|
        \int_B \phi(x;Q) d\hat\rho(x)
        -
        \int_B \phi(x;Q) d\rho(x)
        \right|
        \leq
        4r_2R_C^2
        \sqrt{
            \frac {1}{2m}
            \ln
            \left(
                \frac {2|\cN|^k}{\delta}
            \right)
        }.
        \label{eq:km:basic:2}
    \end{equation}
    For the remainder of this proof, discard the corresponding failure event.

    Now let any $P\in \Hkm_f(\rho;c,k) \cup \Hkm_f(\hat\rho;c,k)$ be given,
    and let $Q\in \cN^k$ be a cover element satisfying
    \cref{eq:km:basic:1} for $P\cap C$.
    By \cref{eq:km:basic:1}, \cref{eq:km:basic:2}, and
    \Cref{fact:km:outer_bracket:2} (and thus discarding an additional failure
    event having probability $2\delta$),
    \begin{align*}
        \left|
        \int \phi_f(x;P)d\rho(x)
        - \int \phi_f(x;P)d\hat\rho(x)
        \right|
        &\leq
        \left|
        \int \phi_f(x;P)d\rho(x)
        - \int_B \phi_f(x;P\cap C)d\rho(x)
        \right|
        %&\textup{\Cref{fact:km:outer_bracket}}
        \\
        &\qquad+
        \left|
        \int_B \phi_f(x;P\cap C)d\rho(x)
        - \int_B \phi_f(x;Q)d\rho(x)
        \right|
        %&\textup{\Cref{fact:km:cover}}
        \\
        &\qquad+
        \left|
        \int_B \phi_f(x;Q)d\rho(x)
        - \int_B \phi_f(x;Q)d\hat\rho(x)
        \right|
        %&\textup{Hoeffding}
        \\
        &\qquad+
        \left|
        \int_B \phi_f(x;Q)d\hat\rho(x)
        - \int_B \phi_f(x;P\cap C)d\hat\rho(x)
        \right|
        %&\textup{\Cref{fact:km:cover}}
        \\
        &\qquad+
        \left|
        \int_B \phi_f(x;P\cap C)d\hat\rho(x)
        - \int \phi_f(x;P)d\hat\rho(x)
        \right|
        %&\textup{\Cref{fact:km:outer_bracket}},
        \\
        &\leq 2\epsilon
        +
        4r_2R_C^2
        \sqrt{
            \frac {1}{2m}
            \ln
            \left(
                \frac {2|\cN|^k}{\delta}
            \right)
        }
        +
        \epsilon_\rho+ \epsilon_{\hat\rho},
    \end{align*}
    and the result follows by unwrapping the definitions of
    $\epsilon_\rho$ and $\epsilon_{\hat\rho}$ from
    \Cref{fact:km:outer_bracket:2}, and $M' = 2^{p'}\epsilon$ as above.
\end{proof}

The more concrete bound for the $k$-means cost is proved as follows.

\begin{proof}[Proof of \Cref{fact:km:basic:kmeans_cost}]
    Set
    \[
        \epsilon := m^{-1/2 + 1/p},
        \qquad
        \qquad
        p' := p/4,
        \qquad
        \qquad
        M' := 2^{p'}\epsilon = 2^{p/4}m^{-1/2+1/p},
    \]
    and recall $f(x) := \|x\|_2^2$ has convexity constants $r_1=r_2=2$.
    Since
    \[
        m
        = \sqrt{m}\sqrt{m}
        \geq \frac {p\sqrt{m}}{2^{p/4+2}e}
        \geq \frac {p'm^{1/2-1/p}}{2^{p'}e}
        = \frac {p'}{M'e}
    \]
    and $p' = p/2 - p/4 \leq p/2 - 1$,
    the conditions for \Cref{fact:km:basic} are met,
    and thus with probability at least $1-\delta$,
    \[
        \left|
        \int \phi_f(x;P)d\rho(x)
        - \int \phi_f(x;P)d\hat\rho(x)
        \right|
        \leq 4\epsilon %m^{-1/2+1/p}
        +
        4R_C^2
        \sqrt{
            \frac {1}{2m}
            \ln
            \left(
                \frac {2|\cN|^k}{\delta}
            \right)
        }
        %+ \sqrt{\frac{2^{p/4}ep}{8m^{3/2-1/p}}} \left(\frac 2 \delta\right)^{4/p},
        + \sqrt{\frac{2^{p/4}ep\epsilon}{8m}} \left(\frac 2 \delta\right)^{4/p},
    \]
    where
    \begin{align*}
        R_C &:=
            (2M)^{1/p} + \sqrt{2c}
            + 2R_B,
        \\
        R_B &:=
        \max\left\{
            (2M)^{1/p}
            + \sqrt{2c}
            ,
            \max_{i\in [p']}
            (M/\epsilon)^{1/(p-2i)}
        \right\},
        \\
        |\cN|
        &\leq
        \left(1 + \frac {2R_Cd}{\tau}\right)^d,
        \\
        \tau
        &:= \min\left\{
            \sqrt{\frac {\epsilon}{4}}
            ,
            \frac {\epsilon}{4(R_B+R_C)}
        \right\}.
    \end{align*}
    To simplify these quantities,
 %  note firstly that
 %  the inner
 %  maximum in the definition of $R_B$ is achieved at either $i=1$ or $i=p'=p/4$,
 %  and secondly
 %  %by separately considering
 %  %the cases $p=4$ and $p\geq 8$ (since $p$ is always a multiple of 4)
 %  that $(p-2)/p^2 = 1/p - 2/p^2 \leq \min\{1/8, 1/p\}$,
 %  whereby
    since $\epsilon \leq 1$, the term $1/\epsilon^{1/(p-2i)}$, as $i$ ranges between
    $1$ and $p-2p'$, is maximized at $i = 1/(p-2p') = 2/p$.
    Therefore, by choice of $M_1$ and $\epsilon$,
    \begin{align*}
        R_B
        &\leq c_1 + (M/\epsilon)^{1/(p-2)}
        + (M/\epsilon)^{1/(p-2p')}
       %\\
       %&
        \leq c_1 + (M^{1/(p-2)}
        + M^{1/(p-2p')})/\epsilon^{2/p}
        \\
        &= c_1 + M_1m^{1/p-2/p^2}
        .
      %%= c_1 + M^{1/(p-2)} m^{1/(2p)}
      %%+ M^{2/p} m^{(p-2)/p^2}.
      %%\\
      %%%&\leq c_1 +M_1 m^{(p-2)/p^2},
      %%&\leq c_1 +M_1 m^{\min\{1/8,1/p\}},
    \end{align*}
    Consequently,
    \begin{align*}
        R_C
        &= c_1 + 2R_B \leq 3c_1 + 2M_1 m^{1/p-2/p^2} % m^{\min\{1/8,1/p\}}
        \qquad
        \textup{and}
        \qquad
        R_C^2
        %&= (c_1 + 2R_B)^2
        %&
        %\leq 18c_1^2 + 8M_1^2 m^{2(p-2)/p^2}.
        %\leq 18c_1^2 + 8M_1^2 m^{\min\{1/4,2/p\}}.
        \leq 18c_1^2 + 8M_1^2 m^{2/p - 4/p^2}.
    %%  &\leq
    %%  2c_1^2 + 8R_B^2
    %%  \leq
    %%  18c_1^2
    %%  + 32\left(
    %%      M^{2/(p-2)} m^{1/p}
    %%      + M^{4/p} m^{2(p-2)/p^2}
    %%  \right)
    %%  \\
    %% %&=
    %% %18c_1^2
    %% %+ 32m^{1/p}\left(
    %% %    M^{2/(p-2)}
    %% %    + M^{4/p} m^{(p-4)/p^2}
    %% %\right),
    %%  &=
    %%  18c_1^2
    %%  + 32M_1 m^{2(p-2)/p^2}.
    \end{align*}
    This entails
    \begin{align*}
        \frac {2R_Cd}{\tau}
        &\leq
        2R_C d\left(
            2m^{1/4-1/(2p)} + 4(R_B + R_C)m^{1/2-1/p}
        \right)
        \\
        &\leq
        8d\left(
            (
            %3c_1 + 2M_1m^{(p-2)/p^2}
            %%3c_1 + 2M_1 m^{\min\{1/8,1/p\}}
            3c_1 + 2M_1 m^{1/p-2/p^2}
            )m^{1/4 - 1/(2p)}
            + (
            %36c_1^2 + 16M_1^2m^{2(p-2)/p^2}
            %36c_1^2 + 16M_1^2 m^{\min\{1/4,2/p\}}
            36c_1^2 + 16M_1^2 m^{2/p-4/p^2}
            )m^{1/2-1/p}
        \right)
        \\
        &\leq
        288d m (c_1 + c_1^2 + M_1 + M_1^2).
    \end{align*}
    Secondly,
    \[
        \frac{R_C^2}{\sqrt{m}}
        \leq (18c_1^2 + 8M_1^2 m^{2/p - 4/p^2})m^{-1/2}
        \leq m^{\min\{1/4,-1/2+2/p\}} (18c_1^2 + 8M_1^2).
    \]
    The last term is direct, since
    \[
        \sqrt{\epsilon/m}
        = m^{-1/4+1/(2p) - 1/2}
        = m^{-1/2 + 1/(2p)}m^{-1/4}.
    \]
    Combining these pieces,
    the result follows.
\end{proof}

\section{Deferred Material from \Cref{sec:km:clamp}}

First, the deferred proof that outer brackets give rise to clamps.

\begin{proof}[Proof of \Cref{fact:bracket_gives_clamp}]
    Throughout this proof, let $\nu$ refer to either $\rho$ or $\hat\rho$,
    with $\epsilon_\nu$ similarly referring to either $\epsilon_\rho$
    or $\epsilon_{\hat\rho}$.
    Let $P \in \Hkm_f(\rho;c,k) \cup \Hkm_f(\hat\rho;c,k)$ be given.

    One direction is direct:
    \begin{align*}
        \int \phi_f(x;P)d\nu(x)
        &\geq \int \phi_f(x;P\cap C)d\nu(x)
        \\
        &\geq \int \min\{\phi_f(x;P\cap C), R\}d\nu(x).
    \end{align*}

    %XXX I am wasting probability here: the outer bracket lemma also invokes
    %XXX the single center compactness thing.. ugh
    For the second direction,
    with probability at least $1-\delta$,
    \Cref{fact:km:one_center:compact} grants the existence of $p'\in P\cap C_0
    \subseteq P\cap C$.  Consequently, for any $x \in B$,
    \begin{align*}
        \min_{p \in P} \sfB_f(x,p)
        &\leq \min_{p \in P\cap C} \sfB_f(x,p)
        \leq \sfB_f(x,p')
        \\
        &\leq r_2\|x-p'\|^2
        \leq 2r_2\left(\|x-\bbE_\rho(X)\|^2 +\|p'-\bbE_\rho(X)\|^2\right)
        \\
        &\leq R;
    \end{align*}
    in other words, if $x\in B$,
    then $\min\{\phi_f(x; P\cap C) , R\} = \phi_f(x; P\cap C)$.
    Combining this with the last part of \Cref{fact:km:outer_bracket:2}.
    \begin{align*}
        \int \min\{\phi_f(x;P\cap C), R\}d\nu(x)
        &\geq
        \int_B \min\{\phi_f(x;P\cap C), R\}d\nu(x)
        \\
        &\geq
        \int_B \phi_f(x;P\cap C)d\nu(x)
        \\
        &\geq
        \int \phi_f(x;P)d\nu(x) - \epsilon_\nu.
    \end{align*}
\end{proof}

The proof of \Cref{fact:km:clamp:basic} will
depend on the following uniform covering property of the clamped
cost (which mirrors \Cref{fact:bregman:cover} for the unclamped cost).

\begin{lemma}
    \label{fact:bregman:cover:clamp}
    Let scale $\epsilon > 0$,
    clamping value $R_3$,
    parameter set $C$ contained within a $\|\cdot\|$-ball of some radius $R_2$,
    and differentiable convex function $f$ with Lipschitz gradient parameter $r_2$
    and strong convexity modulus $r_1$
    with respect to norm $\|\cdot\|$
    be given.
    Define resolution parameter
    \[
        \tau := \min\left\{
            \sqrt{\frac {\epsilon}{2r_2}}
            ,
            \frac {r_1\epsilon}{2r_2R_3}
        \right\},
    \]
    and let $\cN$ be set of centers for a cover of $C$ by $\|\cdot\|$-balls
    of radius $\tau$ (see \Cref{fact:cover:lp_balls} for an estimate when
    $\|\cdot\|$ is an $l_p$ norm).
    It follows that there exists a uniform cover $\cF$ at scale
    $\epsilon$ with cardinality
    $|\cN|^k$, meaning for any collection $P = \{p_i\}_{i=1}^l$
    with $p_i \in C$ and $l\leq k$, there is a cover element $Q$
    with
    \[
        \sup_x\left|
        \min\left\{R_3, \min_{p\in P}\sfB_f(x,p)\right\}
        - \min\left\{R_3, \min_{q \in Q}\sfB_f(x,q)\right\}
        \right|
        \leq \epsilon.
    \]
\end{lemma}
\begin{proof}
    Given a collection $P$ as specified, choose $Q$ so that for every
    $p\in P$, there is $q\in Q$ with $\|p-q\| \leq \tau$, and vice versa.

        First suppose $\min_{q\in Q} \sfB_f(x,q) \geq R_3$; then
        \[
            \min\left\{ R_3, \min_{p\in P} \sfB_f(x,p)\right\}
            \leq R_3
            = \min\left\{ R_3, \min_{q\in Q} \sfB_f(x,q)\right\}
        \]
        as desired.

        Otherwise,
        $\min_{q\in Q} \sfB_f(x,q) < R_3$, which
        by the sandwiching property (cf. \Cref{fact:bregman:easy_norms})
        means
        \[
            r_1\|x-Q(x)\| \leq \sfB_f(x,Q(x)) < R_3.
        \]
        By \Cref{fact:km:cover_prelim},
        \begin{align*}
            \min\left\{ R_3, \min_{p\in P} \sfB_f(x,p)\right\}
            &\leq
            \min\left\{ R_3,
                \min_{q\in Q} \sfB_f(x,q) + \sfB_f(Q(x), P(Q(x)))
                + r_2\|x - Q(x)\|\|Q(x) - P(Q(x))\|
            \right\}
            \\
            &\leq
            \min\left\{ R_3,
                \min_{q\in Q} \sfB_f(x,q)
                + r_2 \tau^2
                + r_2 \tau \|x-Q(x)\|
            \right\}
            \\
            &\leq
            \min\left\{ R_3,
                \min_{q\in Q} \sfB_f(x,q)
                + r_2 \tau^2
                + \frac {r_2R_3}{r_1} \tau
            \right\}
            \\
            &\leq
            \min\left\{ R_3,
            \min_{q\in Q} \sfB_f(x,q)\right\}
            + \epsilon.
        \end{align*}
        The reverse inequality is analogous.
\end{proof}

The proof of \Cref{fact:km:clamp:basic} follows.

\begin{proof}[Proof of \Cref{fact:km:clamp:basic}]
    This proof is a minor alteration of the proof of \Cref{fact:km:basic}.

    By \Cref{fact:bregman:cover:clamp}, let $\cN$ be a cover of the set $C$,
    whereby every set of centers $P\subseteq C$ with $|P|\leq k$ has
    a cover element $Q\in \cN^k$ with
    \begin{equation}
        \sup_{x}
        \left|
        \min\{ \phi_f(x;P), R\}
        - \min\{ \phi_f(x;Q), R\}
        \right| \leq \epsilon;
        \label{eq:km:clamp:basic:1}
    \end{equation}
    when $\|\cdot\|$ is an $l_p$ norm,
    \Cref{fact:cover:lp_balls} provides the stated estimate of its size.
    Since $\min\{\phi_f(x;Q), R\} \in [0,R]$,
    it follows by Hoeffding's inequality and a union bound over $\cN^k$
    that with probability at
    least $1-\delta$,
    \begin{equation}
        \sup_{Q\in\cN^k}
        \left|
        \int_B \phi_f(x;Q) d\hat\rho(x)
        -
        \int_B \phi_f(x;Q) d\rho(x)
        \right|
        \leq
        R
        \sqrt{
            \frac {1}{2m}
            \ln
            \left(
                \frac {2|\cN|^k}{\delta}
            \right)
        }.
        \label{eq:km:clamp:basic:2}
    \end{equation}
    For the remainder of this proof, discard the corresponding failure event.

    Now let any $P\in Z$ be given,
    and let $Q\in \cN^k$ be a cover element satisfying
    \cref{eq:km:clamp:basic:1} for $P\cap C$.
    By \cref{eq:km:clamp:basic:1}, \cref{eq:km:clamp:basic:2}, and
    lastly by the definition of clamp,
    \begin{align*}
        \left|
        \int \phi_f(x;P)d\rho(x)
        - \int \phi_f(x;P)d\hat\rho(x)
        \right|
        &\leq
        \left|
        \int \phi_f(x;P)d\rho(x)
        - \int\min\{ \phi_f(x;P\cap C),R\}d\rho(x)
        \right|
        %&\textup{\Cref{fact:km:outer_bracket}}
        \\
        &\qquad+
        \left|
        \int\min\{ \phi_f(x;P\cap C),R\}d\rho(x)
        - \int \min\{\phi_f(x;Q),R\}d\rho(x)
        \right|
        %&\textup{\Cref{fact:km:cover}}
        \\
        &\qquad+
        \left|
        \int\min\{\phi_f(x;Q),R\}d\rho(x)
        - \int \min\{\phi_f(x;Q),R\}d\hat\rho(x)
        \right|
        %&\textup{Hoeffding}
        \\
        &\qquad+
        \left|
        \int\min\{ \phi_f(x;Q),R\}d\hat\rho(x)
        - \int\min\{ \phi_f(x;P\cap C),R\}d\hat\rho(x)
        \right|
        %&\textup{\Cref{fact:km:cover}}
        \\
        &\qquad+
        \left|
        \int \min\{\phi_f(x;P\cap C),R\}d\hat\rho(x)
        - \int \phi_f(x;P)d\hat\rho(x)
        \right|
        %&\textup{\Cref{fact:km:outer_bracket}},
        \\
        &\leq 2\epsilon +  \epsilon_\rho + \epsilon_{\hat\rho}
        +
        R^2
        \sqrt{
            \frac {1}{2m}
            \ln
            \left(
                \frac {2|\cN|^k}{\delta}
            \right)
        }.
    \end{align*}
\end{proof}

\section{Deferred Material from \Cref{sec:mog}}

%   \begin{definition}
%       The class of \emph{Gaussian mixture penalties} $\Smog(\sigma_1, \sigma_2, k)$
%       consists of all functions of the form
%       \[
%           \phi(x; (\alpha, \Theta))
%           := \phi(x; \{(\alpha_i,\theta_i) = (\alpha_i, \mu_i, \varSigma_i)\}_{i=1}^k)
%           := \ln\left(
%               \sum_{i=1}^k \alpha_i p_{\theta_i}(x)
%           \right),
%       \]
%       where $p_{\theta_i}$ denotes a Gaussian density
%       \[
%           p_{\theta_i}(x) =
%           \frac 1 {\sqrt{(2\pi)^d |\varSigma|}}
%               \exp\left(-\frac 1 2 (x-\mu_i)^T \varSigma_i^{-1} (x-\mu_i)\right),
%       \]
%       and $\alpha\geq 0$, $\sum_i \alpha_i = 1$, and $\sigma_1 I \preceq \varSigma \preceq
%       \sigma_2 I$, where $0 < \sigma_1 \leq \sigma_2$.
%       Additionally, when measures $\rho$ and $\hat\rho$ are
%       available, let $\Smog(\sigma_1,\sigma_2,k;c)$ be those maps with cost at least $c$ according
%       to $\rho$ or $\hat\rho$, meaning
%       \[
%           \Smog(\sigma_1,\sigma_2,k;c) := \{ \phi \in \Smog(\sigma_1,\sigma_2,k) :
%               \bbE_\rho(\phi(X)) \leq c \lor \bbE_{\hat\rho}(\phi(X)) \leq c
%           \}.
%       \]
%   \end{definition}

%   The following definitions will be used throughout the proofs.

%   \red{ don't need..
%   \begin{definition}
%       Let $\cG(\sigma_l, \sigma_u, R)$ denote the set of Gaussian parameters
%       $(\mu,\varSigma)$ with $\|\mu-\bbE_\rho(X)\|_2\leq R$
%       and $\sigma_l I \preceq \varSigma \sigma_u I$, where $\rho$ will be provided by
%       context.
%   \end{definition}
%   }

The following notation for restricting a Gaussian mixture to a certain set of means
will be convenient throughout this section.

\begin{definition}
    Given a Gaussian mixture with parameters $(\alpha, \Theta)$
    (where $\alpha = \{\alpha_i\}_{i=1}^k$
    and $\Theta = \{\theta_i\}_{i=1}^k = \{(\mu_i,\varSigma_i)\}_{i=1}^k$),
    and a set of means $B \subseteq \R^d$,
    define
    \[
        (\alpha, \Theta) \sqcap B
        :=
        \left\{ \left(\{\alpha_i\}_{i\in I}, \{(\mu_i,\varSigma_i)\}_{i \in I}\right)
        : I = \{ 1\leq i \leq k : \mu_i \in B\} \right\}.
    \]
    (Note that potentially $\sum_{i\in I} \alpha_i < 1$, and thus the terminology
    \emph{partial Gaussian mixture} is sometimes employed.)
\end{definition}

\subsection{Constructing an Outer Bracket}

The first step is to show that pushing a mean far away from a region will rapidly decrease
its density there, which is immediate from the condition
$\sigma_1 I \preceq \varSigma \preceq \sigma_2 I$.

%XXX following is used ``for free'' for kmeans since ``obvious'' there.
\begin{lemma}
    \label{fact:mog:separation}
    Let probability measure $\rho$, accuracy $\epsilon > 0$,
    covariance lower bound $0< \sigma_1 \leq \sigma_2$,
    and radius $R$ with corresponding $l_2$ ball
    $B:= \{x\in \R^d : \|x- \bbE_\rho(X)\|_2 \leq R\}$ be given.
    Define
    \begin{align*}
        R_1 &:=\sqrt{
        2\sigma_2 \ln\left(
            \frac {1}{(2\pi\sigma_1)^{d/2}\epsilon^2}
        \right)
        }
        \\
        R_2 &:= R + R_1,
        \\
        B_2
        &:= \{\mu\in \R^d : \|\mu - \bbE_\rho(X)\|_2 \leq R_2\}.
    \end{align*}
    If $\theta = (\mu,\varSigma)$ is the parameterization of a Gaussian density $p_\theta$
    with $\sigma_1 I \preceq \varSigma \preceq \sigma_2 I$
    but $\mu \not \in B_2$,
    then $p_\theta(x) < \epsilon$ for every $x\in B$.
\end{lemma}
\begin{proof}
    Let Gaussian parameters $\theta = (\mu,\varSigma)$ be given with
    $\sigma_1 I \preceq \varSigma \preceq \sigma_2 I$, but
    $\mu \not \in B_2$.
    By the definition of $B_2$, for any $x\in B_1$,
    \begin{align*}
        p_\theta(x)
        < (2\pi\sigma_1)^{-d/2}\exp( - R_1^2/(2\sigma_2) )
        = \epsilon.
    \end{align*}
\end{proof}

The upper component of the outer bracket will be constructed first (and indeed used
in the construction of the lower component).

\begin{lemma}
    \label{fact:mog:upper_outer_bracket_bootstrap}
    Let probability measure $\rho$ with
    order-$p$ moment bound with respect to $\|\cdot\|_2$,
    target accuracy $\epsilon >0$,
    and covariance lower bound $0< \sigma_1$ be given.
    Define
    \begin{align*}
        p_{\max}
        &:= (2\pi \sigma_1)^{-d/2},
        \\
        u(x)
        &:= \ln(p_{\max}),
        \\
        R_u
        &:=
        (M|\ln(p_{\max})| / \epsilon)^{1/p},
        \\
        B_u
        &:=
        \left \{ x \in \R^d : \|x-\bbE_\rho(X)\|_2 \leq R_u \right\}.
    \end{align*}
    If $p_\theta$ denotes a Gaussian density with parameters $\theta=(\mu,\varSigma)$
    satisfying $\varSigma \succeq \sigma_1 I$,
    then $p_\theta \leq u$ everywhere.
    Additionally,
    \[
         \left|\int_{B_u^c} u(x)d\rho(x)\right|
         \leq
         \int_{B_u^c} |u(x)|d\rho(x)
         \leq \epsilon,
    \]
    and with probability at least $1-\delta$ over the draw of $m$ points
    from $\rho$,
    \[
        \left|\int_{B_u^c} u(x)d\hat\rho(x)\right|
        \leq
        \int_{B_u^c} |u(x)|d\hat\rho(x)
        \leq \epsilon +
        |\ln(p_{\max})|\sqrt{\frac {1}{2m} \ln\left(\frac 1 \delta\right)}.
    \]
    (That is to say, $u$ is the upper part of an outer bracket for all Gaussians
    (and mixtures thereof) where each covariance $\varSigma$ satisfies
    $\varSigma \succeq \sigma_1 I$.)
\end{lemma}
\begin{proof}
    Let $p_\theta$ with $\theta = (\mu,\varSigma)$
    satisfying $\varSigma \succeq \sigma_1 I$ be given.  Then
    \[
        p_\theta(x)
        \leq
        \frac {1}{\sqrt{(2\pi)^d \sigma_1^{d}}}
        \exp(\ 0\ )
        = p_{\max}.
    \]
    Next, given the form of $B_u$, if $\ln(p_{\max}) = 0$, the result is immediate,
    thus suppose $\ln(p_{\max}) \neq 0$;
    \Cref{fact:moments:ball}
    provides that $\rho(B_u) \geq 1 - \epsilon/|\ln(p_{\max})|$, whereby
    \[
        \left|\int_{B_u^c} u(x) d\rho(x)\right|
        \leq \int_{B_u^c} |u(x)| d\rho(x)
        = |\ln(p_{\max})| \rho(B_u^c)
        \leq \epsilon.
    \]
    For the finite sample guarantee, by Hoeffding's inequality,
    \[
        \hat\rho(B_u^c)
        \leq
        \rho(B_u^c) + \sqrt{\frac {1}{2m} \ln\left(\frac 1 \delta\right)}
        \leq
        \frac {\epsilon}{|\ln(p_{\max})|}
        + \sqrt{\frac {1}{2m} \ln\left(\frac 1 \delta\right)},
    \]
    which gives the result similarly to the case for $\rho$.
\end{proof}

From, here, a tiny control on $\Smog(\nu;c,k,\sigma_1,\sigma_2)$
emerges, analogous to
\Cref{fact:km:one_center:compact} for $\Hkm_f(\nu;c,k)$.

\begin{lemma}
    \label{fact:mog:one_center:compact}
    Let covariance bounds $0 < \sigma_1 \leq \sigma_2$,
    %with $\sigma_1 \leq 1/(2\pi)$, %XXX meh leave this one out for now
    %$\delta \in (0,1]$, %lolwut breaking protocol
    cost $c \leq 1/2$, %XXX this simplification important since can bake in $c$ dueto not
    %having two forms to p_0
    and probability measure $\rho$
    with order-$p$ moment bound $M$ with respect to $\|\cdot\|_2$
    be given.
    Define
    \begin{align*}
        p_{\max} &:= (2\pi \sigma_1)^{-d/2},
        \\
        R_3 &:=  (2M|\ln(p_{\max})|)^{1/p},
        \\
        R_4 &:= (2M)^{1/p},
        \\
      %%p_0 &:= \min\left\{
      %%1/2,
      %%\exp(4c) / (8e),
      %%p_{\max}
      %%\right\}
      %%\\
      %%\sigma_3 &:= \frac {1}{p_0^2 (2\pi)^d \sigma^{d-1}},
      %%\\
        R_5 &:=
        %\sqrt{ 2\sigma_3 \ln\left( \frac 1 {p_0(2\pi\sigma)^{d/2}}\right)}.
        \sqrt{ 2\sigma_2 \left(
        \ln\left( \frac {8e} {(2\pi\sigma_1)^{d/2}}\right) -4c\right)}.
        \\
        R_6
        &:= \max\{R_3, R_4\} + R_5.
        \\
        B_6 &:= \{x\in \R^d : \|x-\bbE_\rho(X)\|_2 \leq R_6\}.
    \end{align*}
    Suppose
    \begin{align*}
        m &\geq 2\ln(1/\delta) \max\{ 4 , |\ln(p_{\max})|^2 \}.
    \end{align*}
    With probability at least $1-2\delta$,
    given any 
    $(\alpha,\Theta) \in
    \Smog(\rho;c,k,\sigma_1,\sigma_2)
    \cup \Smog(\hat\rho;c,k,\sigma_1,\sigma_2)$,
    the restriction $(\alpha',\Theta') = (\alpha,\Theta) \sqcap B_6$
    is nonempty, and moreover satisfies
    $\sum_{\alpha_i\in \alpha'} \alpha_i \geq \exp(4c) / (8ep_{\max})$.
\end{lemma}

\begin{proof}
    Define
    \begin{align*}
        B_3 := \left\{x\in \R^d : \|x-\bbE_\rho(X)\|_2 \leq \max\{R_3,R_4\}\right\}.
    \end{align*}
    Since $B_3$ has radius at least $R_4$,
    \Cref{fact:moments:ball} provides
    \[
        \rho(B_3) \geq 1/2,
    \]
    and Hoeffding's inequality and the lower bound on $m$ provide (with probability at
    least $1-\delta$)
    \[
        \hat\rho(B_3)
        \geq \frac 1 2 - \sqrt{\frac 2 m \ln \left(\frac 1 \delta\right)}
        > \frac 1 4.
    \]
    Additionally, since $B_3$ also has radius at least $R_3$,
    by \Cref{fact:mog:upper_outer_bracket_bootstrap},
    the choice of $B_3$, and the lower bound on $m$,
    and letting $B_4$ denote the ball of radius $R_3$,
    \[
        \left| \int_{B_3^c} ud\rho\right|
        \leq \int_{B_4^c} |u|d\rho
        \leq \int_{B_4^c} |u|d\rho
         \leq 1/2
        \qquad\textup{and}\qquad
        \left|\int_{B_3^c} ud\hat\rho
        \right| < 1,
    \]
    where the statement for $\hat\rho$ is with probability at least $1-\delta$.
    For the remainder of the proof, let $\nu$ refer to either $\rho$ or $\hat \rho$,
    and discard the $2\delta$ failure probability of either of the above two events.

    For convenience, define $p_0 := \exp(4c)/(8e)$, whereby
    \[
        R_5 = \sqrt{2\sigma_2\ln\left(\frac{1}{p_0 (2\pi\sigma_1)^{d/2}}\right)}.
    \]
    By \Cref{fact:mog:separation}, any Gaussian parameters
    $\theta = (\mu,\varSigma)$ with $\sigma_1I \preceq \Sigma \preceq \sigma_2 I$
    and $\mu \not \in B_6$
    have $p_\theta(x) < p_0$ everywhere on $B_3$.  As such,
    a mixture $(\alpha, \Theta)$ where each $\theta_i\in \Theta$ satisfies these conditions
    also satisfies
    \begin{align*}
        \int \ln\left( \sum_i \alpha_i p_{\theta_i}\right) d\nu
        &\leq
        \int_{B_3} \ln\left(\sum_{
            (\alpha_i,\theta_i) \in (\alpha,\Theta)\sqcap B_6
        }
            \alpha_i p_{\theta_i}
        +\sum_{(
        \alpha_i,\theta_i) \not\in (\alpha,\Theta)\sqcap B_6
    }
    \alpha_i p_{\theta_i}\right) d\nu
        + \int_{B_3^c} u d\nu
        \\
        &<
        \ln\left(\sum_{
            (\alpha_i,\theta_i) \in (\alpha,\Theta)\sqcap B_6
        } \alpha_i p_{\max}
            +\sum_{
        \alpha_i,\theta_i) \not\in (\alpha,\Theta)\sqcap B_6
} \alpha_i p_0\right)\nu(B_3)
        +1
    \end{align*}
    Suppose contradictorily that
    $(\alpha,\Theta)\sqcap B_6=\emptyset$
    or $\sum_{(\alpha_i,\theta_i) \in (\alpha,\Theta)\sqcap B_6} \alpha_i
    < p_0 / p_{\max}$.
    But $c \leq 1/2$ implies $p_0 \leq 1/2$ and so $\ln(2p_0) \leq 0$, thus
    $\ln(2p_0) \nu(B_3) \leq \ln(2p_0)/4$
    which together with $p_0 \leq \exp(4c) / (8e)$ and the above display gives
    \begin{align*}
        \int \ln\left( \sum_i \alpha_i p_{\theta_i}\right) d\nu
        < \ln(2p_0)/4 + 1
        \leq c,
    \end{align*}
    which contradicts $\bbE_\nu(\phig(X;(\alpha,\Theta))) \geq c$.
\end{proof}

Now that significant weight can be shown to reside in some restricted region,
the outer bracket and its basic properties follow (i.e., the analog
to \Cref{fact:km:outer_bracket:2}).

\begin{lemma}
    \label{fact:mog:outer_bracket:2}
    Let target accuracy $0 < \epsilon \leq 1$, %XXX doing this simplification now..
    %XXX really need the output of this lemma to not be complicated.
    covariance bounds $0< \sigma_1\leq \sigma_2$
    with $\sigma_1 \leq 1$, %XXX yeah using this to simplify R_C
    target cost $c$,
    confidence parameter $\delta\in (0,1]$,
    probability measure $\rho$ with order-$p$ moment bound $M$ with
    respect to $\|\cdot\|_2$ with $p\geq 4$,
    and integer $1 \leq p' \leq p/2-1$.
    Define first the basic quantities
    %\red{shold probably bake in $M'$ to watch out for stray $\epsilon$}
    \begin{align*}
        M' &:= 2^{p'}\epsilon,
        \\
        p_{\max}
        &:= (2\pi \sigma_1)^{-d/2},
        \\
  %%    p_0
  %%    &:=% \min\left\{
  %%   % 1/2,
  %%    \exp(4c) / (8e),
  %%   % p_{\max}
  %%   % \right\},
  %%    \\
      % \sigma_3
      % &:= \frac{1}{p_0^2(2\pi)^d\sigma^{d-1}},
      % \\
        R_6
        &:=
            (2M|\ln(p_{\max})|)^{1/p}
            +
            (2M)^{1/p}
            +
            %\sqrt{2\sigma_3\ln(1/ (p_0(2\pi\sigma)^{d/2}))},
            \sqrt{2\sigma_2\left(\ln\left(\frac{8e}{(2\pi\sigma_1)^{d/2}}\right)-4c\right)},
        \\
        B_6 &:= \{x\in \R^d : \|x-\bbE_\rho(X)\|_2 \leq R_6\|\}.
     %% \cG_6
     %% &:=
     %% \cG\left(\sigma, \sigma_3,
     %%     R_6
     %%  %  (2M|\ln(p_{\max})|)^{1/p}
     %%  %  +
     %%  %  (2M)^{1/p}
     %%  %  +
     %%  %  \sqrt{2\sigma_3\ln(1/ (p_0(2\pi\sigma)^{d/2}))}
     %% \right),
     %% \\
        %\\
    \end{align*}
    Additionally define the outer bracket elements
    \begin{align*}
        Z_\ell
      %%&:= \left\{
      %%    %(\alpha,\Theta) : \Theta \in \cG_6, \sum_i \alpha_i \geq p_0 / p_{\max}
      %%    (\alpha,\Theta) : \Theta \in B_6, \sum_i \alpha_i \geq p_0 / p_{\max}
      %%\right\},
        &:= \left\{
            %(\alpha,\Theta) : \Theta \in \cG_6, \sum_i \alpha_i \geq p_0 / p_{\max}
        (\alpha,\Theta) : \forall (\alpha_i, (\mu_i,\theta_i))\in (\alpha,\Theta)\centerdot
        \mu_i \in B_6, \sigma_1 I \preceq \varSigma \preceq \sigma_2 I,
        \sum_i \alpha_i \geq \exp(4c) / (8ep_{\max})
        \right\},
        \\
        c_\ell
        %&:= \ln(p_0/p_{\max}) - \frac d 2 \ln(2\pi\sigma_2),
        &:= 4c - \ln(8ep_{\max}) - \frac d 2 \ln(2\pi\sigma_2),
        \\
        \ell(x) &:= c_\ell - \frac {2}{\sigma_1}\|x-\bbE_\rho(X)\|_2^2,
        \\
        u(x) &:= \ln(p_{\max}),
        \\
        \epsilon_{\hat\rho}
        &:=
        \epsilon + (|c_\ell| + |\ln(p_{\max})|)
        \sqrt{\frac 1 {2m} \ln\left(\frac 1 \delta\right)}
        + \sqrt{\frac {M'ep'}{2m}} \left(\frac 2 \delta\right) ^ {1/p'},
        \\
      %%R_B
      %%&:=
      %%\max\left\{ (2M|c_\ell|/\epsilon)^{1/p},
      %%    (4M\sigma_1 / \epsilon)^{1/(p-2)},
      %%    \max_{1\leq i \leq p'} (M/\epsilon)^{1/(p-2i)},
      %%    (M|\ln(p_{\max})|/\epsilon)^{1/p},
      %%    R_6
      %%\right\},
        M_1 &:=(2M|c_\ell|)^{1/p}
            +
            (4M\sigma_1)^{1/(p-2)}
            +
            \max_{1\leq i \leq p'} M^{1/(p-2i)}
            +
            (M|\ln(p_{\max}))^{1/p},
            \\
        R_B
        &=
        R_6 + M_1 / \epsilon^{1/(p-2p')},
        \\
      %%\\
      %%&=
      %%R_6
      %%+
      %%\left((2M|c_\ell|)^{1/p}
      %%    +
      %%    (4M\sigma_1)^{1/(p-2)}
      %%    +
      %%    \max_{1\leq i \leq p'} M^{1/(p-2i)}
      %%    +
      %%    (M|\ln(p_{\max}))^{1/p}
      %%\right) / \epsilon^{p-2p'}
      %%\\
%
%
      %%\\
        B &:= \{x\in \R^d : \|x-\bbE_\rho(X)\|_2 \leq R_B\}.
 %%     p_{\textup{tiny}}
 %%     &:= \frac {1}{(2\pi\sigma_2)^{d/2}}
 %%     \exp\left(
 %%         - \frac {R_B^2 + R_6^2}{\sigma}
 %%     \right),
 %%     \\
 %%     p_{\textup{smallest}}
 %%     &:=
 %%     \epsilon p_{\textup{tiny}}p_0 / p_{\max},
       %\\
       %\sigma_4
       %&:= \frac{1}{p_{\textup{smallest}}^2(2\pi)^d\sigma^{d-1}},
%       \\
      %%C &:= \cG\left(
      %%    \sigma, \sigma_4, R_B
      %%    + \sqrt{2\sigma_4
      %%        \ln \left(\frac 1 {p_{\textup{smallest}}^2(2\pi)^d \sigma^{d-1}} \right)
      %%    }
      %%\right)
  %%    R_C &:= R_B + \sqrt{2\sigma_2
  %%        \ln \left(\frac 1 {p_{\textup{smallest}}^2(2\pi)^d \sigma_1^{d-1}} \right)
  %%    },
%       R_C &:=
%           R_B(1 + \sqrt{8\sigma_2/\sigma_1})
%               +\sqrt{4\sigma_2\ln(1/\epsilon)}
%               + \sqrt{2\sigma_2
%                   \left(
%                       \ln \left(\frac {64e^2(2\pi\sigma_2)^d}
%                           {(2\pi)^d p_{\max}^4}
%                       \right)
%                       -8c
%                   \right)
%               },
%       \\
%       C &:= \{\mu \in \R^d : \|x-\bbE_\rho(X)\|_2 \leq R_C \}.
    \end{align*}
    The following statements hold with probability at least $1-4\delta$ over a draw
    of size
    \[
        m \geq \max\left\{ p'/ (M'e),
        8\ln(1/\delta),
        2|\ln(p_{\max})|^2\ln(1/\delta)
        \right\}.
    \]
    \begin{enumerate}
        \item
            $(u,\ell)$ is an outer bracket for $\rho$ at scale $\epsilon_\rho := \epsilon$
            with sets $B_\ell := B_u := B$,
            center set class $Z_\ell$ as above,
            and $Z_u = \Smog(\rho; \infty, k, \sigma_1, \sigma_2)$.
           %With probability at least $1-3\delta$ over a random draw of size $m$
           %(with size lower bounded as above),
           %the pair
            Additionally,
            $(u,\ell)$
            is also an outer bracket for $\hat\rho$ at scale $\epsilon_{\hat\rho}$
            with the same sets.
        \item
            Define
            \begin{align*}
                R_C &:=
                1+ %XXX this 1 is here to ensure R_C \geq 1 which simplifies final bound
                R_B(1 + \sqrt{8\sigma_2/\sigma_1})
                +\sqrt{4\sigma_2\ln(1/\epsilon)}
                + \sqrt{2\sigma_2
                    \left(
                        \ln \left(\frac {64e^2(2\pi\sigma_2)^d}
                            {(2\pi)^d p_{\max}^4}
                        \right)
                        -8c
                    \right)
                },
                \\
                C &:= \{\mu \in \R^d : \|x-\bbE_\rho(X)\|_2 \leq R_C \}.
            \end{align*}
            %Every $\phig(\cdot;(\alpha,\Theta))\in \Smog(\sigma_1,\sigma_2,k;c)$,
            Every $(\alpha,\Theta) \in \Smog(\rho;c,k,\sigma_1,\sigma_2)\cup
            \Smog(\hat\rho;c,k,\sigma_1,\sigma_2)$
            satisfies
            $\sum_{(\alpha_i,\theta_i) \in (\alpha,\Theta)\sqcap C} \alpha_i \geq
            \exp(4c)/(8ep_{\max})$, % \red{(need to handle probability mass)},
            and
            \[
                \left|
                \int \phig(x;(\alpha, \Theta)) d\rho(x)
                -
                \int_B \phig(x;(\alpha,\Theta) \sqcap C) d\rho(x)
                \right|
                \leq \epsilon_\rho = 2\epsilon
            \]
            and
          %%with probability at least $1-4\delta$ over a draw of size
          %%$m$ (with size lower bounded as above),
            \[
                \left|
                \int \phig(x;(\alpha,\Theta)) d\hat\rho(x)
                -
                \int_B \phig(x;(\alpha, \Theta) \sqcap C) d\hat\rho(x)
                \right|
                \leq  \epsilon + \epsilon_{\hat\rho}
                .
            \]
    \end{enumerate}
\end{lemma}
\begin{proof}[Proof of \Cref{fact:mog:outer_bracket:2}]
    It is useful to first expand the choice of $R_B$, which was chosen large enough
    to carry a collection of other radii.
    In particular, since $\epsilon \leq 1$, then $1/\epsilon \geq 1$, and therefore
    $1/\epsilon^a \leq 1/\epsilon^b$ when $a \leq b$.  As such,
    since $p' \leq p/2-1$,
    \begin{align*}
        R_B
        &=
        R_6 + M_1 / \epsilon^{1/(p-2p')}
        \\
        &=
        R_6
        +
        \left((2M|c_\ell|)^{1/p}
            +
            (4M\sigma_1)^{1/(p-2)}
            +
            \max_{1\leq i \leq p'} M^{1/(p-2i)}
            +
            (M|\ln(p_{\max}))^{1/p}
        \right) / \epsilon^{1/(p-2p')}
        \\
        &\geq
        R_6
        +
        \left( (2M|c_\ell|/\epsilon)^{1/p}
            +
            (4M\sigma_1 / \epsilon)^{1/(p-2)}
            +
            \max_{1\leq i \leq p'} (M/\epsilon)^{1/(p-2i)}
            +
            (M|\ln(p_{\max})|/\epsilon)^{1/p}
        \right).
    \end{align*}
    Since every term is nonnegative, $R_B$ dominates each individual term.

    \begin{enumerate}
        \item
            The upper bracket and its guarantees were provided by
            \Cref{fact:mog:upper_outer_bracket_bootstrap};
            note that $\epsilon_{\hat\rho}$ is defined large enough to include
            the deviations there,
            and similarly
                $R_B \geq
                (M|\ln(p_{\max})| / \epsilon)^{1/p}$
            means the $B$ here is defined large enough to contain the $B_u$
            there; correspondingly, discard a failure event with probability mass
            at most $\delta$.

      %%    \red{
      %%        therefore need
      %%        $R_B \geq
      %%        (M|\ln(p_{\max})| / \epsilon)^{1/p}$.

      %%        looks like also need $\epsilon_{\hat\rho}$ but maybe try to
      %%        simplify it.
      %%    }

      %%    \red{
      %%        due to following, stick in statement:
      %%        \begin{align*}
      %%            c_\ell
      %%            &:= \ln(p_0/p_{\max}) - \frac d 2 \ln(2\pi\sigma_2)
      %%            = 4c - \ln(8ep_{\max}) - \frac d 2 \ln(2\pi\sigma_2)
      %%            \\
      %%            \ell(x) &:= c_\ell - \frac {2}{\sigma_1}\|x-\bbE_\rho(X)\|_2^2,
      %%            \\
      %%        \end{align*}
      %%        meaning also need $p_{\max}$ there,but I think should bake in $p_0$.
      %%        not sure $c_\ell$ neg or pos matters.
      %%    }

            Let the lower bracket be defined as in the statement; note that its
            properties are much more conservative as compared with the upper bracket.
            Let $(\alpha,\Theta) \in Z_\ell$ be given.
            For every $\theta_i = (\mu_i,\varSigma_i)$,
            $\|\mu_i - \bbE_\rho(X)\|_2 \leq R_6$,
            whereas $R_B \geq R_6$ meaning $x \in B^c$ implies
            $\|x-\bbE_\rho(X)\|_2 \geq R_6$,
            so
            \[
                \|x-\mu_i\|_2
                \leq \|x-\bbE_\rho(X)\|_2 + \|\mu_i - \bbE_\rho(X)\|_2
                \leq 2\|x-\bbE_\rho(X)\|_2,
            \]
            which combined with
            $\sigma_1 I \preceq \varSigma_i \preceq \sigma_2 I$
            gives
            \begin{align*}
                \ln\left(\sum_i \alpha_i p_{\theta_i}(x)\right)
                &\geq
                \ln\left(
                    \sum_i \alpha_i
                    \frac 1 {(2\pi \sigma_2)^{d/2}}
                    \exp\left( - \frac {1}{2\sigma_1} \|x-\mu_i\|_2^2\right)
                \right)
                \\
                &\geq
                \ln(p_0/p_{\max})
                -\frac d 2 \ln(2\pi \sigma_2)
                - \frac {2}{\sigma_1} \|x-\bbE_\rho(X)\|_2^2
                \\
                &= \ell(x),
            \end{align*}
            which is the dominance property.

            Next come the integral properties of $\ell$.
            By \Cref{fact:moments:ball} and since
            $R_B \geq (2M|c_\ell|/\epsilon)^{1/p}$,
            \[
                \left|\int_{B^c} c_\ell d\rho\right|
                \leq
                \int_{B^c} |c_\ell|d\rho
                \leq
                \int_{B^c} |c_\ell|d\rho
                = \rho(B^c) |c_\ell|
                \leq \epsilon/2.
            \]
            Similarly, by Hoeffding's inequality, with probability at least $1-\delta$,
            \[
                \left|\int_{B_\ell^c} c_\ell d\hat\rho\right|
                \leq \epsilon/2 + |c_\ell|\sqrt{\frac {1}{2m} \ln\left(\frac 1 \delta
                \right)}.
            \]
            Now define
            \[
                \ell_1(x)
                := - \frac {2}{\sigma_1} \|x-\bbE_\rho(X)\|_2^2
                = \ell(x) - c_\ell.
            \]
            By \Cref{fact:moment_clipping} and
            since $R_B \geq (4\sigma_1 M/\epsilon)^{1/(p-2)}$,
            \[
                \left|\int_{B^c} \ell_1d\rho\right|
                \leq
                \int_{B^c} |\ell_1|d\rho
                =
                \frac 2 {\sigma_1}
                \int_{B^c} \|x-\bbE_\rho(X)\|_2^2 d\rho(x)
                \leq \epsilon/2.
            \]
            Furthermore by \Cref{fact:moment_clipping:sample} and the above estimate,
            and since
            $R_B \geq \max_{1\leq i \leq p'} (M/\epsilon)^{1/(p-2i)}$ (where
            the maximum is attained at one of the endpoints),
            then with probability at least $1-\delta$
            \[
                \left|\int_{B^c} \ell_1d\hat\rho\right|
                \leq
                \frac \epsilon 2
                + \sqrt{\frac {M' e p'}{2m}} \left(\frac 2 \delta\right)^{1/p'}.
            \]
            Unioning together the above failure probabilities, the general
            controls for $\ell= c_\ell+\ell_1$ follow by the triangle inequality
            and definition of $\epsilon_{\hat\rho}$.

        \item
            Throughout the following, let $\nu$ denote either $\rho$ or $\hat\rho$,
            and correspondingly let $\epsilon_\nu$ respectively refer
            to $\epsilon_\rho$ or $\epsilon_{\hat\rho}$;  let the above
            bracketing properties hold throughout (with events appropriately
            discarded for $\hat\rho$).  Furthermore, for convenience, define
            \[
                p_0 := \exp(4c)/(8e).
            \]

       %%%  \red{
       %%%      the following needs $R_6$.    For now it can be
       %%%      \begin{align*}
       %%%          R_6
       %%%          &:=
       %%%          (2M|\ln(p_{\max})|)^{1/p}
       %%%          +
       %%%          (2M)^{1/p}
       %%%          +
       %%%          \sqrt{2\sigma_2\ln(1/ (p_0(2\pi\sigma_1)^{d/2}))}
       %%%          \\
       %%%          &:=
       %%%          (2M|\ln(p_{\max})|)^{1/p}
       %%%          +
       %%%          (2M)^{1/p}
       %%%          +
       %%%          \sqrt{2\sigma_2\left(
       %%%              \ln((8e)/ (2\pi\sigma_1)^{d/2}) -4c\right)},
       %%%      \end{align*}
       %%%  }

            Let any $(\alpha,\Theta)$ be given with
            %$\phig(\cdot;(\alpha,\Theta))\in \Smog(\sigma_1,\sigma_2,k;c)$.
            $(\alpha,\Theta)\in \Smog(\rho;c,k,\sigma_1,\sigma_2)
            \cup \Smog(\hat\rho;c,k,\sigma_1,\sigma_2)$.
            Define the two index sets
            \begin{align*}
                I_C
                &:= \{ i \in [k] : (\alpha_i, \theta_i) \in (\alpha,\Theta)\sqcap C\},
                \\
                I_6
                &:= \{ i \in [k] : (\alpha_i, \theta_i) \in (\alpha,\Theta)\sqcap B_6\}.
            \end{align*}
            By \Cref{fact:mog:one_center:compact},
            with probability at least $1-\delta$,
            $\sum_{i\in I_6} \alpha_i \geq p_0/p_{\max}$;
            henceforth discard the corresponding failure event, bringing the total
            discarded probability mass to $4\delta$.

            To start, since $\ln(\cdot)$ is concave and thus
            $\ln(a + b) \leq \ln(a) + b/a$ for any positive $a,b$,
            \begin{align*}
                \int \ln\left(
                    \sum_i \alpha_i p_{\theta_i}(x)
                \right)d\nu(x)
                &\leq
                \int_B \ln\left(
                    \sum_i \alpha_i p_{\theta_i}(x)
                \right)d\nu(x)
                +
                \int_{B^c} u(x) d\nu(x)
                \\
                &\leq
                \int_B \ln\left(
                    \sum_{i\in I_C} \alpha_{i} p_{\theta_i}(x)
                \right)d\nu(x)
                +
                \int_B \frac {\sum_{i\not \in I_C}\alpha_i p_{\theta_i}(x)}
                {\sum_{i \in I_C} \alpha_i p_{\theta_i}(x)}
                d\nu(x)
                +
                \epsilon_\nu.
            \end{align*}
            In order to control the fraction, both the numerator and denominator
            will be uniformly controlled for every $x\in B$, whereby the result follows
            since $\nu$ is a probability measure (i.e., the integral is upper bounded with
            an upper bound on the numerator times $\nu(B)\leq 1$
            divided by a lower bound on the denominator).

            For the purposes of controlling this fraction, define
            \begin{align*}
                p_{1}
                &:= \frac {1}{(2\pi\sigma_2)^{d/2}}
                \exp\left(
                    - \frac {R_B^2 + R_6^2}{\sigma}
                \right),
                \\
                p_{\textup{2}}
                &:=
                \epsilon p_{\textup{1}}p_0 / p_{\max},
            \end{align*}
            Observe, by choice of $R_C$ and since $\sigma_1 \leq 1$, that
            \begin{align*}
                R_B + \sqrt{2\sigma_2
                    \ln \left(\frac 1 {p_{\textup{2}}^2(2\pi)^d \sigma_1^{d-1}} \right)
                }
                &\leq
                R_B + \sqrt{2\sigma_2
                    \ln \left(\frac {64e^2p_{\max}^2(2\pi\sigma_2)^d\exp(2(R_B^2+R_6^2))}
                    {\epsilon^2\exp(8c)(2\pi)^d \sigma_1^{d}} \right)
                }
                \\
                &\leq
                R_B
                + \sqrt{2\sigma_2
                    \left(
                        \ln \left(\frac {64e^2(2\pi\sigma_2)^d}
                            {\epsilon^2(2\pi)^d p_{\max}^4}
                        \right)
                        -8c - 4R_B^2/\sigma
                    \right)
                }
                \\
                &\leq
                R_B
                + \sqrt{2\sigma_2
                    \left(
                        \ln \left(\frac {64e^2(2\pi\sigma_2)^d}
                            {(2\pi)^d p_{\max}^4}
                        \right)
                        -8c
                    \right)
                }
                \\
                &\qquad
                +\sqrt{4\sigma_2\ln(1/\epsilon)}
                + R_B\sqrt{8\sigma_2/\sigma_1}
                \\
                &\leq R_C.
                \end{align*}

 %%         \red{following needs (call it $p_1,p_2$ ?)
 %%             \begin{align*}
 %%                 p_{\textup{tiny}}
 %%                 &:= \frac {1}{(2\pi\sigma_2)^{d/2}}
 %%                 \exp\left(
 %%                     - \frac {R_B^2 + R_6^2}{\sigma}
 %%                 \right),
 %%                 \\
 %%     p_{\textup{smallest}}
 %%     &:=
 %%     \epsilon p_{\textup{tiny}}p_0 / p_{\max},
 %%             \end{align*}
 %%             therefore need
 %%             \begin{align*}
 %%     R_C &:= R_B + \sqrt{2\sigma_2
 %%         \ln \left(\frac 1 {p_{\textup{smallest}}^2(2\pi)^d \sigma_1^{d-1}} \right)
 %%     }
 %% \end{align*}
 %% crucial that I un-expand and clean this one here.
 %%         }

            For the denominator,
            first note for every $x\in B$ and parameters $\theta = (\mu,\varSigma)$
            with $\sigma_1 I \preceq \varSigma \preceq \sigma_2 I$
            and $\mu \in B_6$ that
            \begin{align*}
                p_{\theta}(x)
                &\geq
                \frac {1}{(2\pi\sigma_2)^{d/2}}
                \exp\left(-\frac {1}{2\sigma_1}\|x-\mu\|_2^2\right)
                \\
                &\geq
                \frac {1}{(2\pi\sigma_2)^{d/2}}
                \exp\left(
                    -\frac {1}{2\sigma_1}
                    \left(\|x-\bbE_\rho(X)\|_2+\|\bbE_\rho(X) - \mu_i\|_2\right)^2
                \right)
                \\
                &\geq p_{\textup{1}}.
            \end{align*}
            Consequently, for $x\in B$,
            \[
                \sum_{i\in I_C} \alpha_i p_i(x)
                \geq \sum_{i\in I_6} \alpha_i p_i(x)
                \geq p_{\textup{1}}\sum_{i\in I_6} \alpha_i
                \geq p_{\textup{1}}p_0 / p_{\max}.
            \]
            For the numerator, by choice of $C$ (as developed above with
            the definitions of $p_1$ and $p_2$) and
            an application of \Cref{fact:mog:separation},
            for $p_i$ corresponding to $i\not \in I_C$,
            \[
                p_i(x) \leq \epsilon p_1p_0 / p_{\max}
                = p_{\textup{2}}.
            \]
            It follows that the fractional term is at most $\epsilon$,
            which gives the first direction of the desired inequality.

            To get the other direction,
            since $\sum_{i\in I_6} \alpha_i \geq p_0/p_{\max}$ due to
            \Cref{fact:mog:one_center:compact} as discussed above,
            it follows that $(\alpha,\Theta)\sqcap B_6 \in Z_\ell$,
            meaning the corresponding partial Gaussian mixture can be controlled
            by $\ell$.  As such, since $R_6\leq R_B$ thus $I_6 \subseteq I_C$,
            and since $\ln$ is nondecreasing,
            \begin{align*}
                \int_B \ln \left(\sum_{i\in I_C} \alpha_i p_i\right)d\nu
                &=
                \int \ln \left(\sum_{i\in I_C} \alpha_i p_i\right)d\nu
                -
                \int_{B^c} \ln \left(\sum_{i\in I_C} \alpha_i p_i\right)d\nu
                \\
                &\leq
                \int \ln \left(\sum_{i\in I_C} \alpha_i p_i\right)d\nu
                -
                \int_{B^c} \ln \left(\sum_{i\in I_6} \alpha_i p_i\right)d\nu
                \\
                &\leq
                \int \ln\left( \sum_{i\in I_C} \alpha_i p_i\right)d\nu
                -
                \int_{B^c} \ell d\nu
                \\
                &\leq
                \int \ln \left(\sum_{i\in I_C} \alpha_i p_i\right)d\nu
                +
                \epsilon_\nu
                \\
                &\leq
                \int \ln \left(\sum_{i} \alpha_i p_i\right)d\nu
                +
                \epsilon_\nu.
            \end{align*}
    \end{enumerate}
\end{proof}

\subsection{Uniform Covering of Gaussian Mixtures}

First, a helper \namecref{fact:mog:cover:covariance} for covering covariance matrices.

\begin{lemma}
    \label{fact:mog:cover:covariance}
    Let scale $\epsilon > 0$
    and eigenvalue bounds $0 < \sigma_1 \leq \sigma_2$
    be given.
    There exists a subset $\cM$  of the positive definite
    matrices satisfying $\sigma_1 I \preceq M\preceq \sigma_2 I$
    so that
    \[
        |\cM|
        \leq
       %(1 + 4/\tau)^{d^2}
       %\left(
       %    \left(\frac{\sigma_2-\sigma_1}{\tau'}\right)^d
       %    + \left(\frac{\ln(\sigma_2/\sigma_1)}{\ln(\tau'')}\right)^d
       %\right),
        (1 + 32\sigma_2/\epsilon)^{d^2}
        \left(
            \left(1+\frac{\sigma_2-\sigma_1}{\epsilon/2}\right)^d
            + \left(\frac{\ln(\sigma_2/\sigma_1)}{\epsilon/d}\right)^d
        \right),
    \]
    and for any $A$ with $\sigma_1I \preceq A \preceq \sigma_2 I$,
    there exists $B\in \cM$ with
    \[
        \exp(-\epsilon) \leq \frac {|A|}{|B|} \leq \exp(\epsilon)
        \qquad
        \textup{and}
        \qquad
        \|A - B\|_2 \leq \epsilon.
    \]
\end{lemma}
\begin{proof}
    The mechanism of the proof is to separately cover the set of orthogonal matrices
    and the set of possible eigenvalues; this directly leads to the determinant control,
    and after some algebra, the spectral norm control follows as well.
    %XXX lelz
   %The approach
   %is elementary and the constants could be improved.

    With foresight, set the scales
    \begin{align*}
        \tau
        &:= \epsilon / (8\sigma_2)
        ,
        \\
        \tau'
        &:= \epsilon/2
        ,
        \\
        \tau''
        &:= \exp(\epsilon/d)
        .
    \end{align*}

    First, a cover of the orthogonal $d\times d$ matrices at scale $\tau$ is
    constructed as follows.  The entries of these orthogonal matrices are within
    $[-1,+1]$, thus
    first construct a cover $Q'$ of all matrices $[-1,+1]^{d\times d}$ at scale
    $\tau/2$ according to the maximum-norm, which simply measures the max among
    entrywise differences; this cover can be constructed by gridding each coordinate
    at scale $\tau/2$, and thus
    \[
        |Q'| \leq (1 + 4/\tau)^{d^2}.
    \]
    Now, to produce a cover of the orthogonal matrices, for each $M'\in Q'$,
    if it is within max-norm distance $\tau/2$ of some orthogonal matrix $M$, include
    $M$ in the new cover $Q$; otherwise, ignore $M'$.  Since $Q'$ was a max-norm cover
    of $[-1,+1]^{d\times d}$ at scale $\tau/2$, then $Q$ must be a max-norm
    cover of the orthogonal matrices at scale $\tau$ (by the triangle inequality),
    and it still holds that
    \[
        |Q| \leq (1 + 4/\tau)^{d^2}.
    \]
    Since the max-norm is dominated by the spectral norm, for any orthogonal matrix
    $O$, there exists $Q\in M$ with $\|O-Q\|_2 \leq \tau$.

    Second, a cover of the set of possible eigenvalues is constructed as follows; since
    both a multiplicative and an additive guarantee are needed for the eigenvalues,
    two covers will be unioned together.
    First, produce a cover $L_1$ of
    the set $[\sigma_1, \sigma_2]^d$ at scale $\tau'$ entrywise as usual,
    which means $|L_1| \leq (1+(\sigma_2-\sigma_1)/\tau')^d$.  Second, the cover
    $L_2$ will cover each coordinate multiplicatively, meaning each coordinate
    cover consists of $\sigma_1, \sigma_1\tau'', \sigma_1(\tau'')^2$, and so on;
    consequently, this cover has size $|L_2| \leq \ln(\sigma_2/\sigma_1) / \ln(\tau'')$.
    Together, the cover $L := L_1 \cup L_2$ has size
    \[
    |L| \leq \left(1+\frac{\sigma_2-\sigma_1}{\tau'}\right)^d
        + \left(\frac{\ln(\sigma_2/\sigma_1)}{\ln(\tau'')}\right)^d,
    \]
    and for any vector $\Lambda \in [\sigma_1,\sigma_2]^d$, there exists $\Lambda' \in L$
    with
    %XXX something about this is making me twitchy.  maybe
    %XXX i should only be taking an element larger than the referenced thing? whatever.
    \[
        \frac {1}{\tau''}\leq \max_i \Lambda'_i/\Lambda_i \leq \tau''
        \qquad\textup{and}\qquad
        \max_i |\Lambda'_i - \Lambda_i| \leq \tau.
    \]
   %(Trying to make a single additive or multiplicative cover satisfying both conditions
   %would drastically increase the cover size.)
    Note there was redundancy in this construction: $L$ need only contain nondecreasing
    sequences.

    The final cover $\cM$ is thus the cross product of $Q$ and $L$, and correspondingly
    its size is the product of their sizes.
    Given any $A$ with $\sigma_1 I \preceq A \preceq \sigma_2 I$ with
    spectral decomposition $O_1^\top \Lambda_1 O_1$, pick a corresponding
    $O_2\in Q$ which is closest to $O_1$ in spectral norm, and $\Lambda_2\in L$ which is
    closest to $\Lambda_1$ in max-norm, and set $B = O_2^\top \Lambda_2 O_2$.
    By the multiplicative guarantee on $L$,
    it follows that
    \[
        \left(\frac 1 {\tau''}\right)^d
        \leq \frac {|\Lambda_2|}{|\Lambda_1|}
        = \frac {|B|}{|A|} \leq (\tau'')^d;
    \]
    by the choice of $\tau''$, the determinant guarantee follows.
    Secondly, relying on a few properties of spectral norms
    ($\|XY\|_2 \leq \|X\|_2\|Y\|_2$ for square matrices, and $\|Z\|_2=1$ for
    orthogonal matrices, and of course the triangle inequality),
    \begin{align*}
        \|A-B\|_2
        &=\left \|(O_1-O_2+O_2)^\top \Lambda_1 (O_1-O_2+O_2)^\top -
        O_2^\top\Lambda_2 O_2\right\|_2
        \\
        &\leq
        \|O_2^\top\Lambda_1 O_2 - O_2^\top \Lambda_2 O_2\|_2
        + 2 \|O_2^\top \Lambda_1 (O_1-O_2)\|_2
        + \|(O_1-O_2)^\top \Lambda_1 (O_1-O_2)\|_2
        \\
        &\leq
        \|\Lambda_1 - \Lambda_2\|_2
        + 2 \|O_1-O_2\|_2 \|\Lambda_1\|_2
        + \|O_1-O_2\|_2  \|\Lambda_1\|_2 (\|O_1\|_2 + \|O_2\|_2)
        \\
        &\leq
        \tau' + 4 \tau \sigma_2,
    \end{align*}
    and the second guarantee follows by choice of $\tau$ and $\tau'$.
\end{proof}

The covering \namecref{fact:mog:cover} is as follows.

\begin{lemma}
    \label{fact:mog:cover}
    Let scale $\epsilon > 0$,
    ball $B := \{ x\in \R^d : \|x- \bbE(X)\| \leq R\}$,
    mean set $X := \{ x\in \R^d : \|x-\bbE(X)\| \leq R_2\}$,
    covariance eigenvalue bounds $0 < \sigma_1 \leq \sigma_2$,
    mass lower bound $c_1>0$,
    and number of mixtures $k > 0$ be given.
    Then there exists a cover set $\cN$ (where $(\mu,\varSigma)\in \cN$
    has $\mu \in X$ and $\sigma_1 I \preceq \varSigma \preceq \sigma_2 I$)
    of size
    \[
        |\cN|
        \leq
        \left(
            \left(
                \frac {\ln(1/\alpha_0)}{\ln(\tau_0)}
                + \frac {1-\alpha_0}{\tau_4}
            \right)
            \ \cdot\ 
                \left(1 + \frac {2R_2d}{\tau_1}\right)^d
            \ \cdot\ 
                (1 + 32/(\sigma_1\tau_2))^{d^2}
                \left(
                    \left(1+\frac{\sigma_1^{-1} -\sigma_2^{-1}}{\tau_2/2}\right)^d
                    + \left(\frac{\ln(\sigma_2/\sigma_1)}{\tau_2/d}\right)^d
                \right)
        \right)^k
    \]
    where
    \begin{align*}
        \tau_0
        &:= \exp(\epsilon/4),
        \\
        \tau_1
        &:=
        \min\left\{
            \frac {\epsilon\sigma_1}{16(R+R_2)}
            ,
            \sqrt{\frac {\epsilon\sigma_1}{8}}
        \right\},
        \\
        \tau_2
        &:= \frac {\epsilon}{4\max\{1, (R+R_2)^2\}},
        \\
        p_{\min}
        &:= \frac{1}{(2\pi \sigma_2)^{d/2}}
        \exp(-(R+R_2)^2  /(2\sigma_1)),
        \\
        p_{\max}
        &:= (2\pi\sigma_1)^{-d/2},
        \\
        \alpha_0
        &:=
        \frac {\epsilon c_1 p_{\min}}{4k(p_{\max} + \epsilon p_{\min}/2)},
        \\
        \tau_4
        &:=
        \alpha_0,
    \end{align*}
    (whereby $p_{\min} \leq p_\theta(x) \leq p_{\max}$ for $x\in B$ and
    $\theta=(\mu,\varSigma)$ satisfies $\mu\in X$ and $\sigma_1 I\preceq \Sigma\preceq
    \sigma_2I$,)
    so that for every partial Gaussian mixture
    $(\alpha,\Theta) = \{(\alpha_i,\mu_i,\varSigma_i)\}$
    with $\alpha_i \geq 0$,
    $c_1 \leq \sum_i \alpha_i \leq 1$,
    $\mu_i \in X$, and $\sigma_1 I \preceq \varSigma_i
    \preceq \sigma_2 I$ there is an element $(\alpha',\Theta')\in \cN$
    with weights $c_1 - k\alpha_0 \leq \sum_i \alpha_i' \leq 1$ so
    that, for every $x\in B$,
    \[
        |\ln(p_{\alpha,\Theta}(x)) - \ln(p_{\alpha',\Theta'}(x))| \leq \epsilon.
    \]
\end{lemma}
\begin{proof}
    The proof controls components in two different ways.  For those
    where the weight $\alpha_i$ is not too small, both $\alpha_i$ and
    $p_{\theta_i}$ are closely (multiplicatively) approximated.
    When $\alpha_i$ is small, its contribution can be discarded.  Between
    these two cases, the bound follows.

   %To start, with foresight, set various scale parameters as
   %\begin{align*}
   %    \tau_0
   %    &:= \exp(\epsilon/4)
   %    \\
   %    \tau_1
   %    &:=
   %    \min\left\{
   %        \frac {\epsilon\sigma_1}{16(R+R_2)}
   %        ,
   %        \sqrt{\frac {\epsilon\sigma_1}{8}}
   %    \right\}
   %    \\
   %    \tau_2
   %    &:= \frac {\epsilon}{4\max\{1, (R+R_2)^2\}},
   %    \\
   %    p_{\min}
   %    &:= \frac{1}{(2\pi \sigma_2)^{d/2}}
   %    \exp(-(R+R_2)^2  /(2\sigma_1))
   %    \\
   %    p_{\max}
   %    &:= (2\pi\sigma_1)^{-d/2},
   %    \\
   %    \alpha_0
   %    &:=
   %    \frac {\epsilon c_1 p_{\min}}{4k(p_{\max} + \epsilon p_{\min}/2)}
   %    \\
   %    \tau_4
   %    &:=
   %    \alpha_0.
   %\end{align*}
    Note briefly that for any $\theta = (\mu,\varSigma)$ with $\mu\in X$
    and $\sigma_1 I \preceq \varSigma \preceq \sigma_2 I$,
    \begin{align*}
        p_{\theta}(x)
        &\leq \frac{1}{(2\pi \sigma_1)^{d/2}} \exp(\ 0\ ) = p_{\max},
        \\
        p_{\theta}(x)
        &\geq \frac{1}{(2\pi \sigma_2)^{d/2}}
        \exp(-\|x-\mu\|_2^2  /(2\sigma_1))
        \\
        &\geq \frac{1}{(2\pi \sigma_2)^{d/2}}
        \exp(-(\|x-\bbE_\rho(X)\|_2 + \|\mu-\bbE_\rho(X)\|)^2  /(2\sigma_1))
        \\
        &= p_{\min}.
    \end{align*}

    Next, the covers of each element of the Gaussian mixture are as
    follows.
    \begin{enumerate}
        \item Union together
            a multiplicatively grid of $[\alpha_0,1]$ at scale $\tau_0$ (meaning
            produce a sequence of the form $\alpha_0, \alpha_0\tau_0,
            \alpha_0\tau_0^2$, and so on),
            and an additive grid of $[\alpha_0,1]$ at scale $\tau_4$;
            together, the grid has a size of at most
            \[
                \frac {\ln(1/\alpha_0)}{\ln(\tau_0)}
                + \frac {1-\alpha_0}{\tau_4}.
            \]
        \item
            Grid the candidate center set $X$ at scale $\tau_1$,
            which by \Cref{fact:bregman:cover} can be done with size
            at most
            \[
                \left(1 + \frac {2R_2d}{\tau_1}\right)^d.
            \]
        \item
            Lastly, grid the \emph{inverse} of covariance matrices (sometimes
            called precision matrices),
            meaning $\sigma_2^{-1} I \preceq \varSigma^{-1} \preceq \sigma_1^{-1}$,
            whereby \Cref{fact:mog:cover:covariance} grants that a
            cover of size
            \[
                (1 + 32/(\sigma_1\tau_2))^{d^2}
                \left(
                    \left(\frac{\sigma_1^{-1} -\sigma_2^{-1}}{\tau_2/2}\right)^d
                    + \left(\frac{\ln(\sigma_1/\sigma_2)}{\tau_2/d}\right)^d
                \right)
            \]
            suffices to provide that for any permissible $\varSigma^{-1}$,
            there exists a cover element $A$ with
            \[
                \exp(-\tau_2) \leq \frac {|\varSigma^{-1}|}{|A|} \leq \exp(\tau_2)
                \qquad
                \textup{and}
                \qquad
                \|\varSigma^{-1} - A\|_2 \leq \tau_2.
            \]
    \end{enumerate}
    Producting the size of these various covers and raising to the power $k$ (to
    handle at most $k$ components), the cover size in the statement is met.

    Now consider a component $(\alpha_i,\mu_i,\varSigma_i)$ with $\alpha_i \geq \alpha_0$;
    a relevant cover element $(a_i, c_i, B_i)$ is chosen as follows.
    \begin{enumerate}
        \item
            Choose the largest $a_i \leq \alpha_i$ in the gridding
            of $[\alpha_0,1]$, whereby it follows that
            $\sum_i a_i \leq \sum_i \alpha_i \leq 1$,
            and also
            \[
                \tau_0^{-1} \leq a_i/\alpha_i \leq \tau_0
                \qquad\textup{and}\qquad
                a_i \geq \alpha_i - \tau_4.
            \]
            Thanks to the second property,
            \[
                \sum_{\alpha_i\geq \alpha_0} a_i
                \geq
                \left(\sum_{\alpha_i\geq \alpha_0} \alpha_i\right) - k\tau_4.
            \]
        \item
            Choose $c_i$ in the grid on $X$ so that $\|\mu_i - c_i\| \leq \tau_1$.
        \item
            Choose covariance $B_i$ so that
            \[
                \exp(-\tau_2) \leq \frac {|B_i|}{|\varSigma_i|} \leq \exp(\tau_2)
                \qquad
                \textup{and}
                \qquad
                \|\varSigma^{-1} - B_i^{-1}\|_2 \leq \tau_2.
            \]
            The first property directly controls for the determinant term in the
            Gaussian density.  To control the Mahalanobis term, note that the above
            display, combined with $\|\mu_i - c_i\|\leq \tau_1$, gives, for
            every $x\in B$,
            \begin{align*}
                &
                \left|
                (x-\mu_i)^\top \varSigma^{-1}_i (x-\mu_i)
                -
                (x-c_i)^\top B^{-1}_i (x-c_i)
                \right|
                \\
                &=
                \left|
                (x-\mu_i)^\top \varSigma^{-1}_i (x-\mu_i)
                -
                (x-c_i)^\top (B^{-1}_i -\varSigma^{-1}_i + \Sigma^{-1}_i) (x-c_i)
                \right|
                \\
                &\leq
                \left|
                (x-\mu_i)^\top \varSigma^{-1}_i (x-\mu_i)
                -
                (x-c_i)^\top \varSigma^{-1}_i (x-c_i)
                \right|
                + \|x-c_i\|_2^2\|B^{-1}_i - \varSigma^{-1}_i\|_2
                \\
                &\leq
                \left|
                (x-\mu_i)^\top \varSigma^{-1}_i (x-\mu_i)
                -
                (x-c_i)^\top \varSigma^{-1}_i (x-c_i)
                \right|
                + (R + R_2)^2 \tau_2
                \\
                &\leq
                \left|
                (x-\mu_i)^\top \varSigma^{-1}_i (x-\mu_i)
                -
                (x-c_i)^\top \varSigma^{-1}_i (x-c_i)
                \right|
                + \epsilon/4.
            \end{align*}
            Continuing with the (still uncontrolled) first term,
            \begin{align*}
                &
                \left|
                (x-\mu_i)^\top \varSigma^{-1}_i (x-\mu_i)
                -
                (x-c_i)^\top \varSigma^{-1}_i (x-c_i)
                \right|
                \\
                &=
                \left|
                (x-\mu_i)^\top \varSigma^{-1}_i (x-\mu_i)
                -
                (x-\mu_i + \mu_i - c_i)^\top \varSigma^{-1}_i (x-\mu_i + \mu_i - c_i)
                \right|
                \\
                &\leq
                2\|x-\mu_i\|_2\|\mu_i-c_i\|_2\|\varSigma^{-1}\|_2
                +
                \|\mu_i-c_i\|_2^2 \|\varSigma^{-1}\|_2
                \\
                &\leq
                2(R+R_2) \tau_1 /\sigma_1
                + \tau_1^2 / \sigma_1
                \\
                &\leq \epsilon/4.
            \end{align*}
    \end{enumerate}
    Combining these various controls with the choices of scale parameters,
    for some provided probability $\alpha_i p_i$ and cover element probability $a_i p'_i$,
    it follows for $x\in B$ that
    \[
        \exp(-3\epsilon/4)
        \leq \frac {\alpha_i p_i(x)}{a_i p'_i(x)}
        \leq \exp(3\epsilon/4).
    \]

    Lastly, when $\alpha_i < \alpha_0$, simply do not bother to exhibit a cover element.

    To show $|\ln(p_{\alpha,\Theta}(x)) - \ln(p_{\alpha',\Theta'}(x))| \leq \epsilon$, consider
    the two directions separately as follows.
    \begin{enumerate}
        \item
            Given the various constructions above, since $\ln$ is nondecreasing,
            \begin{align*}
                \ln\left(
                    \sum_i a_i p_{\theta'_i}(x)
                \right)
                &\leq
                \ln\left(
                    \sum_{\alpha_i \geq \alpha_0} \alpha_i p_{\theta_i}(x)\exp(3\epsilon/4)
                    + \sum_{\alpha_i < \alpha_0} \alpha_i p_{\theta_i}(x)
                \right)
                \\
                &\leq
                \ln\left(
                    \sum_{i} \alpha_i p_{\theta_i}(x)
                \right)
                + \frac {3\epsilon} 4.
            \end{align*}
        \item
            On the other hand,
            \begin{align*}
                \ln\left(
                    \sum_i \alpha_i p_{\theta_i}(x)
                \right)
                &=
                \ln\left(
                    \sum_{\alpha_i \geq \alpha_0} \alpha_i p_{\theta_i}(x)
                    + \sum_{\alpha_i < \alpha_0} \alpha_i p_{\theta_i}(x)
                \right)
                \\
                &\leq
                \ln\left(
                    \sum_{\alpha_i \geq \alpha_0} a_i p_{\theta'_i}(x)\exp(3\epsilon/4)
                    + k\alpha_0 p_{\max}
                \right)
                \\
                &=
                \ln\Bigg(
                    (1+\epsilon/4)
                    \sum_{\alpha_i \geq \alpha_0} a_i p_{\theta'_i}(x)\exp(3\epsilon/4)
                    \\
                    &\qquad
                    -
                    \epsilon/4
                    \sum_{\alpha_i \geq \alpha_0} a_i p_{\theta'_i}(x)\exp(3\epsilon/4)
                    + k\alpha_0 p_{\max}
                    \Bigg).
            \end{align*}
            But since $\sum_i a_i \geq c_1 - k(\tau_4 + \alpha_0)$,
            \begin{align*}
                    -
                    \epsilon/4
                    \sum_{\alpha_i \geq \alpha_0} a_i p_{\theta'_i}(x)\exp(3\epsilon/4)
                    &\leq
                    -
                    (\epsilon/4)
                    (c_1 - k(\tau_4+\alpha_0))
                    p_{\min}
                    + k\alpha_0 p_{\max}
                    \leq 0.
            \end{align*}
            As such, since $(1+\epsilon/4) \leq \exp(\epsilon/4)$, the result
            follows in this case as well.
    \end{enumerate}
\end{proof}

\subsection{Proof of \Cref{fact:mog:basic}}

\begin{proof}[Proof of \Cref{fact:mog:basic}]
    This proof is based on the proof of \Cref{fact:km:basic}.
    Let the various quantities in \Cref{fact:mog:outer_bracket:2} be given;
    in particular, let balls $B,C$ and their radii $R_B,R_C$ be as provided there.
    Additionally, define $p_0 := \exp(4c)/8e$ for convenience.
    Near the end of the proof, the choices $p' = p/4$ and $\epsilon := m^{-1/2+1/p}$
    will be made.

    By \Cref{fact:mog:cover}, let $\cN$ be a cover of the set $C$,
    with all parameters having the same names as those here, except the $R$ there
    is the radius $R_B$ here, and $R_2$ there is radius $R_C$ here,
    the lower bound $c_1$ is $p_0/p_{\max}$.
    By the construction of the cover there,
    every set of partial Gaussian
    parameters $(\alpha,\Theta) \in C$ with $\sum_{\alpha_i} \alpha_i \geq c_1
    = p_0 / p_{\max}$
    and cardinality at most $k$
    has a cover element $Q\in \cN$ with
    \begin{equation}
        \sup_{x\in B}
        \left|
        \phig(x;(\alpha, \Theta)) - \phig(x;Q)
        \right| \leq \epsilon;
        \label{eq:mog:basic:1}
    \end{equation}
    note that \Cref{fact:mog:cover} also provides the stated estimate of the size.
    Next, note for $x\in B$ and every cover element $Q\in \cN$ that
    \Cref{fact:mog:cover} provides
    \[
        \ln((c_1 - k\alpha_0) p_{\min}) \leq  p_Q(x) \leq \ln(p_{\max})
    \]
    where $c_1 = p_0/p_{\max}$ as above
    and
    \[
        \alpha_0
        = \frac {\epsilon c_1 p_{\min}}{4k(p_{\max}+\epsilon p_{\min}/2)}
        \leq \frac {\epsilon c_1 p_{\min}}{4kp_{\max}},
    \]
    which combined with $\epsilon \leq 2$ and $p_{\min} \leq p_{\max}$ gives
    \[
        c_1 - k\alpha_0
        \geq c_1 \left(1- \frac {\epsilon p_{\min}}{4p_{\max}}\right)
        \geq\frac{c_1}{2}.
    \]
    Thus, by Hoeffding's inequality,
    \begin{align}
        \sup_{Q\in\cN}
        \left|
        \int_B \phig(x;Q) d\hat\rho(x)
        -
        \int_B \phig(x;Q) d\rho(x)
        \right|
        &\leq
        \ln\left(\frac{p_{\max}}
            {p_{\min}(c_1 - k\alpha_0)}
        \right)
        \sqrt{
            \frac {1}{2m}
            \ln
            \left(
                \frac {2|\cN|}{\delta}
            \right)
        }
        \notag
        \\
        &\leq
        \ln\left(\frac{2p_{\max}^2}
            {p_{\min}p_0}
        \right)
        \sqrt{
            \frac {1}{2m}
            \ln
            \left(
                \frac {2|\cN|}{\delta}
            \right)
        }.
        \label{eq:mog:basic:2}
    \end{align}
    For the remainder of this proof, discard the corresponding failure event.

    To further simplify \cref{eq:mog:basic:2}, note firstly that
    \begin{align*}
        \ln\left(\frac {1}{p_{\min}}\right)
        &= \ln\left(
        (2\pi \sigma_2)^{d/2}
        \exp((R_B+R_C)^2  /(2\sigma_1))
        \right)
        \\
        &=
        \ln( (2\pi \sigma_2)^{d/2})
        + 2R_C^2/\sigma_1,
    \end{align*}
    where
    \begin{align*}
        R_C^2
        & \leq
        3R_B^2(1 + \sqrt{8\sigma_2/\sigma_1})^2
        +12\sigma_2\ln(1/\epsilon)
        + 6\sigma_2
        \left(
            \ln \left(\frac {64e^2(2\pi\sigma_2)^d}
                {(2\pi)^d p_{\max}^4}
            \right)
            -8c
        \right)
    \end{align*}
    and
    \[
        R_B^2 \leq 2R_6^2 + M_1^2 / \epsilon^{2/(p-2p')}.
    \]
    Next, to control $|\cN|$, the scale term $\tau = \min\{\tau_1,\tau_2\}$
    must first be controlled.
    Since $\epsilon \leq 1$ and $\sigma_1 \leq 1$ and $R_C\geq 1$,
    \begin{align*}
        \tau \geq \frac {\epsilon\sigma_1}{16(R_B+R_C)^2}
        \geq \frac {\epsilon\sigma_1}{64R_C^2},
    \end{align*}
    and thus
    \begin{align*}
        \ln\left(\frac \epsilon \tau\right)
        &\leq
        \ln(64R_C^2/\sigma_1).
    \end{align*}
    Together with $\tau_0 = \exp(\epsilon/4)$ and
    $\alpha_0 \geq \epsilon c_1 p_{\min} / (8kp_{\max}) = p_0 p_{\min} / (8kp_{\max}^2)$,
    and letting $\cO(\cdot)$ swallow terms depending only on numerical constants,
    $c$, $\sigma_1$, and $\sigma_2$, 
    but in particular not touching terms depending on $\epsilon$, $d$, $k$
    or $m$ or $\delta$, %XXX I should keep track of $c$ as well but no time
    \begin{align*}
        \ln(|\cN|)
        &\leq
        \ln
        \left(
        \left(
            \left(
                \frac 5 \epsilon
                \left(
                    %\frac {4kp_{\max}(p_{\max}+\epsilon p_{\min}/2)} {\epsilon p_0p_{\min}}
                    \frac {8kp_{\max}^2} {\epsilon p_0p_{\min}}
                \right)
            \right)
            \left(
                \frac {3R_Cd}{\tau}
            \right)^d
            \left(
                \frac {33}{\sigma_1 \tau}
            \right)^{d^2}
            \left(
                \left(\frac{\sigma_1^{-1}}{\tau/2}\right)^d
                +
                \left(\frac{\ln(\sigma_2/\sigma_1)}{\tau/2}\right)^d
            \right)
        \right)^k
        \right)
        \\
        &\leq
        3d^2k\left(
            5\ln(1/\epsilon)
            + \ln(1/p_{\min})
            + 3\ln(\epsilon/\tau)
            + \ln(R_C)
            +
            \cO(1)
        \right)
        \\
        %XXX yeah, no time for precision now..
      %%&\leq
      %%15d^2k\bigg(
      %%    \ln(1/\epsilon)
      %%   %\\
      %%   %&
      %%    \quad
      %%    \quad +
      %%    \ln\left(\frac{(2\pi)^{d/2}(2\pi\sigma_3)^{d^2/2}}{\epsilon^d p_0^d}\right)
      %%    +dR_B^2/\sigma
      %%    +2R_C^2/\sigma
      %%    \\
      %%    &\quad +
      %%    \ln(64/\sigma) + 2\ln(R_C)
      %%   %\\
      %%   %&
      %%    \quad
      %%    \quad
      %%    + \ln(R_C)
      %%   %\\
      %%   %&
      %%    \quad
      %%    \quad +
      %%    \ln((2\pi\sigma_3)^{d}/(\epsilon^2 p_0^2))
      %%    + 2R_B^2/\sigma
      %%    \\
      %%    &\quad
      %%    + \ln(80000kdp_{\max}^2/(p_0\sigma^3)
      %%\bigg)
      %%\\
      %%&\leq
      %%15d^2k\bigg(
      %%    (d+3)\ln(1/\epsilon)
      %%    + (d+2)R_B^2/\sigma
      %%    +2 2R_C^2/\sigma
      %%    + 3\ln(R_C)
      %%    + (d+d^2/d)\ln(\sigma_3)
      %%    \\
      %%    &\qquad\qquad
      %%    + \ln\left(\frac{(2\pi)^{d/2}(2\pi)^{d^2/2}}{\epsilon^d p_0^d}\right)
      %%    + \ln((2\pi)^{d}/(\epsilon^2 p_0^2))
      %%    + \ln(80000kdp_{\max}^2/(p_0\sigma^3)
      %%\bigg).
        &\leq
        3d^2k\left(
            5\ln(1/\epsilon)
            + 2R_C^2/\sigma_1
            + 3\ln(\epsilon/\tau)
            + 4\ln(R_C)
            +
            \cO(1)
        \right)
        \\
        &= \cO\left(d^2k(\ln(1/\epsilon) + \epsilon^{-2/(p-2p')}\right).
    \end{align*}
    Together, the full expression in
    \cref{eq:mog:basic:2}
    may be simplified down to
    \begin{align}
        &
        \sup_{Q\in\cN}
        \left|
        \int_B \phig(x;Q) d\hat\rho(x)
        -
        \int_B \phig(x;Q) d\rho(x)
        \right|
        \notag\\
        &\qquad\leq
        \cO\left(\textup{poly}(d,k)
            \left(\frac {1}{\epsilon}\right)^{2/(p-2p')}
        \sqrt{
            \frac{(\ln(1/\epsilon) + (1/\epsilon)^{2/(p-2p')} + \ln(1/\delta))}{m}
        }
    \right)
    \notag\\
        &\qquad\leq
    \cO\left(\textup{poly}(d,k)\left(
            %%%%%XXX WAT\frac{\epsilon^{-3/(4p - 8p')}}{\sqrt{m}}
            \frac{\epsilon^{-3/(p - 2p')}}{\sqrt{m}}
        + \sqrt{
            \frac{(\ln(1/\epsilon) + \ln(1/\delta))}{m}
        }
    \right)
    \right)
    \notag\\
        &\qquad\leq
    \cO\left(\textup{poly}(d,k)\left(
            m^{-1/2+3/p}
        + \sqrt{
            \frac{(\ln(m) + \ln(1/\delta))}{m}
        }
    \right)
    \right)
        \label{eq:mog:basic:3}
    \end{align}
    where the final step used the choice $p' = p/4$ and $\epsilon := m^{-1/2+1/p}$.
  %%and the fact that
  %%$(1/2 - 1/p) \cdot (3/(2p)) - 1/2 \leq -1/2 + 1/p$.

    %Now let any $\phig(\cdot;(\alpha,\varSigma))\in \Smog(\sigma,k;c)$ be given,
    Now let any $(\alpha,\Theta) \in \Smog(\rho;c,k,\sigma_1,\sigma_2)
    \cup \Smog(\hat\rho;c,k,\sigma_1,\sigma_2)$ be given,
    and let $Q\in \cN$ be a cover element satisfying
    \cref{eq:mog:basic:1} for $(\alpha,\theta)\sqcap C$.
    By \cref{eq:mog:basic:1}, \cref{eq:mog:basic:2}, and
    \Cref{fact:mog:outer_bracket:2} (and thus discarding an additional failure
    event having probability $4\delta$),
    \begin{align*}
        \left|
        \int \phig(x;(\alpha,\Theta))d\rho(x)
        - \int \phig(x;(\alpha,\Theta))d\hat\rho(x)
        \right|
        &\leq
        \left|
        \int \phig(x;(\alpha,\Theta))d\rho(x)
        - \int_B \phig(x;(\alpha,\Theta)\sqcap C)d\rho(x)
        \right|
        %&\textup{\Cref{fact:km:outer_bracket}}
        \\
        &\qquad+
        \left|
        \int_B \phig(x;(\alpha,\Theta)\sqcap C)d\rho(x)
        - \int_B \phig(x;Q)d\rho(x)
        \right|
        %&\textup{\Cref{fact:km:cover}}
        \\
        &\qquad+
        \left|
        \int_B \phig(x;Q)d\rho(x)
        - \int_B \phig(x;Q)d\hat\rho(x)
        \right|
        %&\textup{Hoeffding}
        \\
        &\qquad+
        \left|
        \int_B \phig(x;Q)d\hat\rho(x)
        - \int_B \phig(x;(\alpha,\Theta)\sqcap C)d\hat\rho(x)
        \right|
        %&\textup{\Cref{fact:km:cover}}
        \\
        &\qquad+
        \left|
        \int_B \phig(x;(\alpha,\Theta)\sqcap C)d\hat\rho(x)
        - \int \phig(x;(\alpha,\Theta))d\hat\rho(x)
        \right|
        %&\textup{\Cref{fact:km:outer_bracket}},
        \\
        &\leq 4\epsilon
        +
        \ln\left(\frac{2p_{\max}^2}
            {p_{\min}p_0}
        \right)
        \sqrt{
            \frac {1}{2m}
            \ln
            \left(
                \frac {2|\cN|}{\delta}
            \right)
        }
        +
        \epsilon_\rho+ \epsilon_{\hat\rho},
        \\
        &=\textup{poly(d,k)}
        \cO(m^{-1/2+3/p}
        \left(
            1+ \sqrt{\ln(m) + \ln(1/\delta)}
            + (1/\delta)^{4/p}
        \right),
    \end{align*}
    where the final step uses the above simplification of the cover term,
    the choices $\epsilon = m^{-1/2+1/p}$ and $p'= p/4$,
    and additionally unwrapping the forms of
    $\epsilon_\rho$ and $\epsilon_{\hat\rho}$ from
    \Cref{fact:mog:outer_bracket:2}.
\end{proof}

\end{document}